\documentclass[nohyperref]{article}

\usepackage{microtype}
\usepackage{graphicx}
\usepackage{booktabs} 

\usepackage{hyperref}
\usepackage{comment}

\usepackage[accepted]{icml2022}

\usepackage{amsmath}
\usepackage{amssymb}
\usepackage{mathtools}
\usepackage{amsthm}

\usepackage[capitalize,noabbrev]{cleveref}

\usepackage{bm}
\usepackage{bbm}
\usepackage[caption=false,font=footnotesize]{subfig}
\usepackage{algorithm}
\usepackage{algorithmic}
\usepackage{multirow}
\usepackage{lipsum}
\usepackage{dirtytalk}

\newcommand{\reals}{{\mathbb{R}}}

\newcommand{\naturals}{{\mathbb{N}}}

\newcommand{\bfu}[1]{{\bm{#1}}}

\newcommand{\mean}[1]{\mathbb{E}\left[{#1}\right]}

\newcommand{\vc}[1]{\mathbf{#1}}
\newcommand{\norm}[1]{\|{#1}\|}
\newcommand{\pnorm}[2]{\|{#1}\|_{#2}}
\newcommand{\tp}[1]{{#1}^{\text{T}}}

\newcommand{\calA}{{\cal A}}

\newcommand{\calI}{{\cal I}}
\newcommand{\calJ}{{\cal J}}
\newcommand{\calF}{{\cal F}}
\newcommand{\calR}{{\cal R}}

\newcommand{\calH}{{\cal H}}
\DeclareMathOperator*{\argmax}{arg\,max}
\DeclareMathOperator*{\argmin}{arg\,min}

\theoremstyle{plain}
\newtheorem{theorem}{Theorem}[section]

\newtheorem{lemma}[theorem]{Lemma}

\theoremstyle{definition}
\newtheorem{definition}[theorem]{Definition}

\theoremstyle{remark}

\newtheorem*{theorem2}{Theorem}
\usepackage[textsize=tiny]{todonotes}

\icmltitlerunning{Adversarial Vulnerability of Randomized Ensembles}

\begin{document}

\twocolumn[
\icmltitle{Adversarial Vulnerability of Randomized Ensembles}



\icmlsetsymbol{equal}{*}

\begin{icmlauthorlist}
\icmlauthor{Hassan Dbouk}{uiuc}
\icmlauthor{Naresh R. Shanbhag}{uiuc}
\end{icmlauthorlist}

\icmlaffiliation{uiuc}{Department of Electrical and Computer Engineering, University of Illinois at Urbana-Champaign, Urbana, USA}

\icmlcorrespondingauthor{Hassan Dbouk}{hdbouk2@illinois.edu}

\icmlkeywords{Machine Learning, ICML}

\vskip 0.3in
]



\printAffiliationsAndNotice{}  

\begin{abstract}
Despite the tremendous success of deep neural networks across various tasks, their vulnerability to imperceptible adversarial perturbations has hindered their deployment in the real world. Recently, works on randomized ensembles have empirically demonstrated significant improvements in adversarial robustness over standard adversarially trained (AT) models with minimal computational overhead, making them a promising solution for safety-critical resource-constrained applications. However, this impressive performance raises the question: \textit{Are these robustness gains provided by randomized ensembles real?} In this work we address this question both theoretically and empirically. We first establish theoretically that commonly employed robustness evaluation methods such as adaptive PGD provide a false sense of security in this setting. Subsequently, we propose a theoretically-sound and efficient adversarial attack algorithm (ARC) capable of compromising random ensembles even in cases where adaptive PGD fails to do so. We conduct comprehensive experiments across a variety of network architectures, training schemes, datasets, and norms to support our claims, and empirically establish that randomized ensembles are in fact \textit{more vulnerable} to $\ell_p$-bounded adversarial perturbations than even standard AT models. Our code can be found at \url{https://github.com/hsndbk4/ARC}.
\end{abstract}

\section{Introduction}
\label{sec:intro}

Deep neural networks (DNNs) are ubiquitous today, cementing themselves as the \emph{de facto} learning model for various machine learning tasks \cite{sarwar2001item,he2016deep,fastrcnn,farhadi2018yolov3,devlin2018bert,zhang2017hello}. Yet, their vulnerability to \emph{imperceptible} adversarial perturbations \cite{szegedy2013intriguing,goodfellow2014explaining,ilyas2019adversarial} has caused major concerns regarding their security, and hindered their deployment in safety-critical applications. 

Various attack algorithms for crafting such perturbations have been proposed in the literature \cite{kurakin2016adversarial,moosavi2016deepfool,carlini2017towards}; the majority of which rely on moving the input along the gradient to maximize a certain loss function. Powerful attacks, such as projected gradient descent (PGD) \cite{madry2018towards}, perform this process iteratively until a perturbation is found, and can successfully fool undefended DNNs causing them to misclassify \emph{every} test sample in the test set.

\begin{figure}[!t]
  \centering
   \includegraphics[width=0.99\columnwidth]{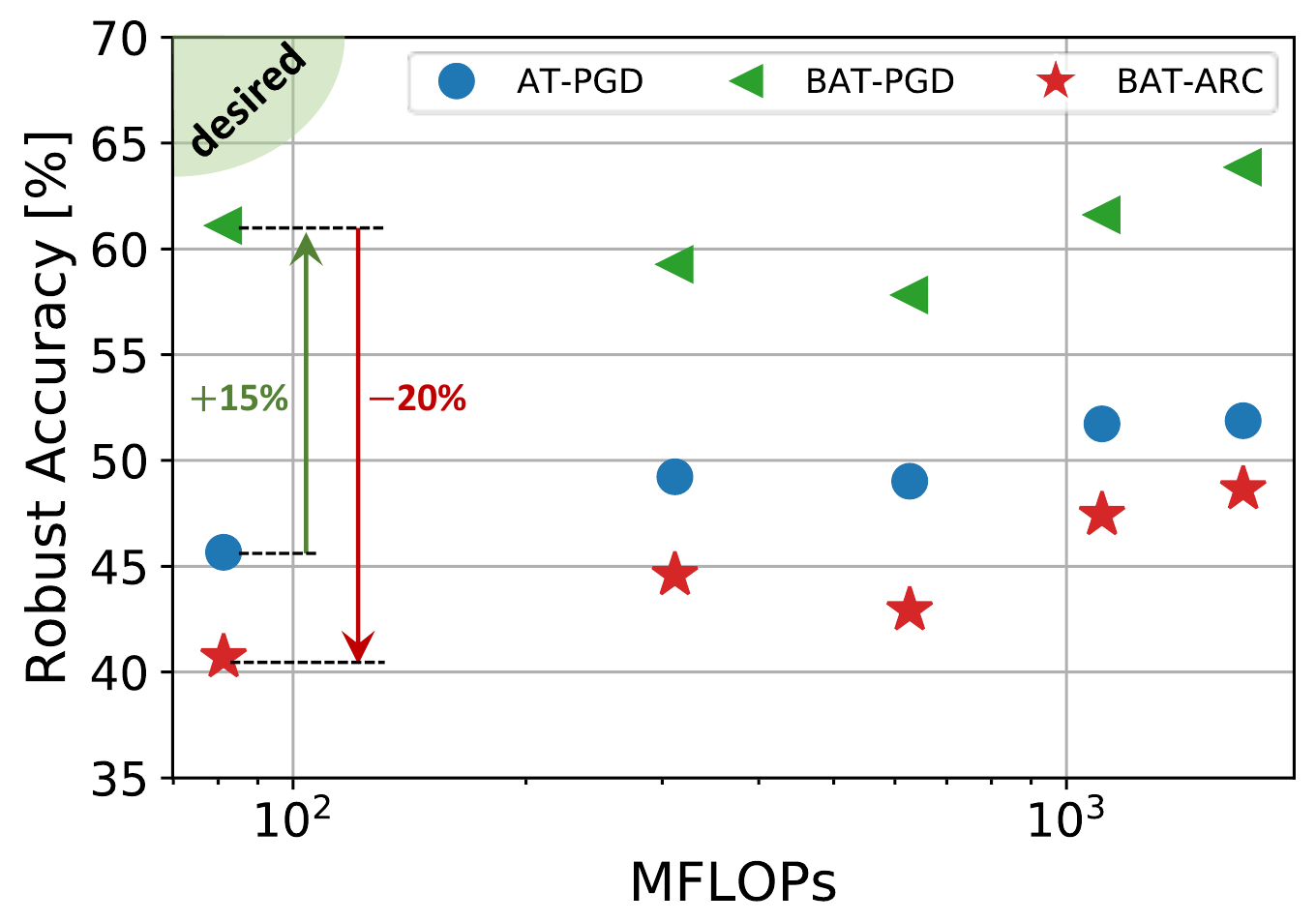}\label{fig:motivation}%
  \caption{The adversarial \emph{vulnerability} of randomized ensembles across various network architectures (from left to right: ResNet-20, MobileNetV1, VGG-16, ResNet-18, and WideResNet-28-4) against $\ell_\infty$ norm-bounded perturbations on CIFAR-10. Compared to a standard adversarially trained (AT) single model (\color{RoyalBlue}$\bullet$\color{black}), a randomized ensemble of two models obtained via BAT (\color{ForestGreen}$\blacktriangleleft$\color{black}) \cite{pinot2020randomization} offers significant improvement in robustness at iso-FLOPs, when evaluated using an adaptive PGD adversary. However, the robust accuracy of the BAT ensemble suffers a massive drop and becomes worse than even the single model AT baseline when evaluated using the proposed ARC (\color{Maroon}$\star$\color{black}) adversary, rendering the ensemble obsolete.} 
\end{figure}

In order to robustify DNNs against adversarial perturbations, several defense strategies have been proposed \cite{cisse2017parseval,pang2019improvingADP,yang2019me,trades,madry2018towards}. While some of them were subsequently broken with more powerful attacks (adaptive attacks) \cite{tramer2020adaptive,athalye2018obfuscated}, adversarial training (AT) \cite{goodfellow2014explaining,trades,madry2018towards} remains the strongest empirical defense that thus far has stood the test of time.

Recently, the work of \cite{pinot2020randomization} showed, from a game-theoretic point of view, that the robustness of any deterministic classifier can be outperformed by a randomized one. Specifically, \cite{pinot2020randomization} 
proposed constructing random classifiers by randomly sampling from an ensemble of classifiers trained via boosted adversarial training (BAT), and empirically demonstrated \emph{impressive} ($\sim$15\%) improvements in robust accuracy over single model standard AT classifiers (see Fig.~\ref{fig:motivation}). Randomized ensembles also present a promising way for achieving high robustness with \emph{limited} compute resources since the number of FLOPs per inference is fixed\footnote{Assuming all the models in the ensemble share the same architecture.} --
a problem that has garnered much interest recently \cite{dbouk2021generalized, sehwag2020hydra, NAS}. However, this recent success of randomized ensembles raises the following question:
\begin{center}
    \emph{Are the robustness gains provided by randomized ensembles real?}   
\end{center}

In this work, we tackle this question both theoretically and empirically. Specifically, we make the following contributions:
\begin{itemize}
    \item 
    We prove that standard attack algorithms such as PGD suffer from a fundamental flaw when employed to attack randomized ensembles of linear classifiers -- there are no guarantees that it will find an $\ell_p$-bounded adversarial perturbation even when one exists (\emph{inconsistency}).
    \item We propose a provably (for linear classifiers) \textit{consistent} and efficient adversarial perturbation algorithm -- the \textbf{A}ttacking \textbf{R}andomized ensembles of \textbf{C}lassifiers (\textbf{ARC}) algorithm -- that is tailored to evaluate the robustness of randomized ensembles against $\ell_p$-bounded perturbations. 
    \item We employ the ARC algorithm to demonstrate empirically that randomized ensembles of DNNs are in fact \textit{more} vulnerable to $\ell_p$-bounded perturbations than standard AT DNNs (see Fig.~\ref{fig:motivation}). 
    \item We conduct comprehensive experiments across a variety of network architectures, training schemes, datasets, and norms to support our observations.
\end{itemize}
Our work suggests the need for improved adversarial training algorithms for randomized ensembles.

\section{Background and Related Work} \label{sec:background}
\textbf{Adaptive Attacks}: Benchmarking new defenses against adaptive attacks, i.e., attacks carefully designed to target a given defense, has become standard practice today, thanks to the works of \cite{athalye2018obfuscated, tramer2020adaptive}. The authors in \cite{athalye2018obfuscated} identify the phenomenon of obfuscated gradients, a type of gradient masking, that leads to a false
sense of security in defenses against adversarial
examples. By identifying and addressing different types of obfuscated gradients such as stochastic gradients due to randomization, they were able to circumvent and fool most defenses proposed at the time. More recently, \cite{tramer2020adaptive} identified and circumvented \emph{thirteen} defenses from top venues (NeurIPS, ICML, and ICLR), despite most employing adaptive attacks, by customizing these attacks to the specific defense. Therefore, it is imperative when evaluating a new defense to not only re-purpose existing adaptive attacks that circumvented a prior defense, but also \say{adapt} these attacks to target the new defense. Our work falls in the category of adaptive attacks for randomized ensembles. Specifically, we identify an unforeseen vulnerability of randomized ensembles by demonstrating the shortcomings of commonly used PGD-based attacks and then proposing an adaptive attack to successfully circumvent them.

\textbf{Ensembles and Robustness}: Traditionally, the use of ensembling techniques such as bagging \cite{breiman1996bagging} and boosting \cite{dietterich2000ensemble, freund1997decision} has been a popular technique for improving the performance of machine learning models. Building on that success, and in order to address the apparent vulnerability of deep nets, a recent line of work \cite{kariyappa2019improvingGAL,pang2019improvingADP,sen2019empir,yang2020dverge,yang2021trs} has proposed the use of DNN ensembles. The intuition is that promoting some form of diversity between the different members of the ensemble would alleviate the adversarial transferability that DNNs exhibit, and as a byproduct improve overall robustness. Specifically, GAL \cite{kariyappa2019improvingGAL} proposes maximizing the cosine distance between the members' gradients, whereas EMPIR \cite{sen2019empir} leverages extreme model quantization to enforce ensemble diversity. ADP \cite{pang2019improvingADP} proposes a training regularizer that promotes different members to have high diversity in the non-maximal predictions. These earlier works have been subsequently broken by adaptive and more powerful attacks \cite{tramer2020adaptive}. More recently, DVERGE \cite{yang2020dverge} presents a robust ensemble training approach that diversifies the non-robust features of the different members via an adversarial training objective function. TRS \cite{yang2021trs}, on the other hand, provides a theoretical framework for understanding adversarial transferability and establishes bounds that show model smoothness and input gradient diversity are both required for low transferability. The authors also propose an ensemble training algorithm to empirically reduce transferability. While none of these works deal with randomized ensembles explicitly, their proposed techniques can be seamlessly adapted to that setting, where our work shows that existing powerful attacks such as PGD can provide a \emph{false} sense of robustness.

\textbf{Randomization and Robustness}: Randomization has been shown, both theoretically and empirically, to have a very strong connection with adversarial robustness. Randomized smoothing \cite{lecuyer2019certified,cohen2019certified} proves a tight robustness guarantee in $\ell_2$ norm for smoothing classifiers with Gaussian noise. Prior to \cite{lecuyer2019certified}, no certified defense has been shown feasible on ImageNet. In the same spirit, SNAP \cite{patil2021robustifying} enhances the robustness of single attack $\ell_\infty$ AT frameworks against the union of perturbation types via shaped noise augmentation. 
Recently, the work of \cite{pinot2020randomization} showed, from a game-theoretic point of view, that the robustness of any deterministic classifier can be outperformed by a randomized one when evaluated against deterministic attack strategies. A follow-up work from \cite{pmlr-v139-meunier21a} removes any assumptions on the attack strategy and further motivates the use of randomization to improve robustness. Building on their framework, the authors of \cite{pinot2020randomization} propose training random ensembles of DNNs using their BAT algorithm and \emph{empirically} demonstrate impressive robustness gains over individual adversarially trained DNNs on CIFAR-10, when evaluated against a strong PGD adversary. However, our work demonstrates that these robustness gains give a \emph{false} sense of security and propose a new attack algorithm better suited for randomized ensembles. \looseness=-1

\section{Preliminaries} \label{sec:preliminaries}
\subsection{Problem Setup}
Let $\calF=\{f_1, f_2, ..., f_M\}$ be a collection of $M \in \naturals$ arbitrary $C$-ary differentiable classifiers $f_i: \reals^D \rightarrow \reals^C$. A classifier $f_i$ assigns the label $m\in[C]$ to data-point $\vc{x} \in \reals^D$ if and only if $\big[f_i(\vc{x})\big]_m > \big[f_i(\vc{x})\big]_j \ \ \forall j\in[C]\setminus\{m\}$,
where $[f_i(\vc{x})]_j$ is the $j^{\text{th}}$ component of $f_i(\vc{x})$. 

A randomized ensemble classifier (REC) $g$ is defined over the tuple $(\calF,\bm{\alpha})$ as: $\Pr\{g(\vc{x}) = f_d(\vc{x})\}= \alpha_d$ where $d\in[M]$ is sampled \textit{independent} of $\vc{x}$. This independence assumption is crucial as the choice of the sampled classifier cannot be tampered by any adversary. \looseness=-1

For any labeled data $(\vc{x},y) \in \reals^D \times [C]$, define $L$ as the expected classification accuracy of $g$:
\begin{align} \label{eq:acc}
    \begin{split}
    L(\vc{x},y,\bm{\alpha}) &= \mean{\mathbbm{1}\left\{\argmax_{j\in[C]}\big[g(\vc{x})\big]_j = y\right\}} \\
    &= \sum_{i=1}^M \alpha_i \mathbbm{1}\left\{\argmax_{j\in[C]}\big[f_i(\vc{x})\big]_j = y\right\}
    \end{split}
\end{align}
where $\mathbbm{1}\left\{.\right\}$ is the indicator function. Note that we have $L(\vc{x},y,\bm{\alpha}) \leq 1$ with equality if and only if all the classifiers in $\calF$ correctly classify $\vc{x}$. Similarly $L(\vc{x},y,\bm{\alpha}) \geq 0$ with equality if and only if all all the classifiers misclassify $\vc{x}$. We always assume that $\alpha_i>0$, otherwise we can remove $f_i$ from the ensemble as it is not used.

\begin{definition}
\label{def:adversarial}
A perturbation $\bm{\delta}\in\reals^D$ is called adversarial to $(\calF,\bm{\alpha})$ for the labeled data-point $(\vc{x},y)$ if:
\begin{equation}
    L(\vc{x}+\bm{\delta},y,\bm{\alpha}) < L(\vc{x},y,\bm{\alpha})  
\end{equation}
\end{definition}
That is, the mean classification accuracy of the randomized ensemble strictly decreases when $\vc{x}$ is perturbed by $\bm{\delta}$. To that end, an adversary's objective is to search for a norm-bounded adversarial perturbation $\bm{\delta}^*$ such that:
\begin{equation}\label{eq:opt}
    \bm{\delta}^* =  \argmin_{\bm{\delta}: \pnorm{\bm{\delta}}{p} \leq \epsilon} L(\vc{x}+\bm{\delta},y,\bm{\alpha})  
\end{equation}
Let $\calA$ be an adversarial perturbation generation algorithm that takes as input an REC $(\calF, \bm{\alpha})$, a labeled data-point $(\vc{x},y)$, and a norm bound $\epsilon$ and generates a perturbation $\bm{\delta}_{\calA}:\pnorm{\bm{\delta}_{\calA}}{p}\leq \epsilon$ while attempting to solve the optimization problem in \eqref{eq:opt}. 

\begin{definition}
\label{def:consistent}
Given REC $(\calF, \bm{\alpha})$, a labeled data-point $(\vc{x},y)$, and a norm bound $\epsilon$: 
An adversarial perturbation generation algorithm $\calA$ is said to be \emph{consistent} if it finds a norm-bounded adversarial $\bm{\delta}_{\calA}$ whenever $L(\vc{x},y,\bm{\alpha})=1$ and a norm-bounded adversarial $\bm{\delta}$ exists.
\end{definition}

Definition~\ref{def:consistent} implies that a consistent algorithm will always find an adversarial perturbation, if one exists, when all the classifiers in the ensemble correctly classify $\vc{x}$. It does not imply that it will find the optimal adversarial perturbation $\bm{\delta}^*$. Conversely, an \emph{inconsistent} adversarial algorithm will fail to find a norm-bounded adversarial perturbation under these conditions and provide an overly optimistic estimate of adversarial robustness. Note that the condition $L(\vc{x},y,\bm{\alpha})=1$ is not too restrictive, since DNNs typically exhibit very high accuracies on clean samples.

\subsection{Projected Gradient Descent}
Optimization-based attacks, such as PGD \cite{madry2018towards, maini2020adversarial}, search for adversarial perturbations around the data sample $(\vc{x},y)$ for some classifier $f$ and loss function $l$, by solving the following maximization:
\begin{equation}
    \bm{\delta}^* = \argmax_{\bm{\delta}:\pnorm{\bm{\delta}}{p} \leq \epsilon} l\big(f(\vc{x}+\bm{\delta}),y\big)
\end{equation}
in an iterative fashion $\forall k \in[K]$:
\begin{equation} \label{eq:pgd}
    \bm{\delta}^{(k)} = \Pi_\epsilon^p \left( \bm{\delta}^{(k-1)} + \eta \mu_p\left(\nabla_{\vc{x}} l(\vc{x}+\bm{\delta}^{(k-1)} ,y)\right)\right)
\end{equation}
where $\bm{\delta}^{(0)}$ is typically a random initial guess inside the $\ell_p$ ball of radius $\epsilon$, $\eta$ is the step size, $\mu_p$ computes the unit-norm $\ell_p$ steepest direction, and $\Pi_\epsilon^p$ computes the projection onto the $\ell_p$ ball of radius $\epsilon$. A typical choice of loss function is the cross-entropy loss \cite{madry2018towards}. To simplify notation, we will use $l(\vc{x}+\bm{\delta},y)$ instead of $l\big(f(\vc{x}+\bm{\delta}),y\big)$. For instance when $p=\infty$, the update equation in \eqref{eq:pgd} reduces to:
\begin{equation} \label{eq:pgd-linf}
    \bm{\delta}^{(k)} = \text{clip}\left( \bm{\delta}^{(k-1)} + \eta \text{sgn}\left(\vc{g}\right),-\epsilon,\epsilon\right)
\end{equation}
where $\vc{g}$ is the gradient of the loss as in \eqref{eq:pgd}.

\textbf{Adaptive PGD for REC}: As pointed out by \cite{athalye2018obfuscated, tramer2020adaptive}, randomized defenses need to take caution when evaluating robustness. In particular, one should use the expectation over transformation (EOT) method to avoid gradient masking. The discrete nature of our setup allows for exact computation of expectations, and eliminates the need for Monte Carlo sampling \cite{pinot2020randomization}. Therefore, in order to solve the optimization in \eqref{eq:opt} and evaluate the robustness of randomized ensembles, \cite{pinot2020randomization} adapted the PGD update  rule \eqref{eq:pgd} by replacing the loss function with its expected value $ \mean{l(\vc{x}+\bm{\delta},y)}$, which can be computed efficiently. Interestingly, we find that it is in fact this adaptation of PGD that leads to an overly optimistic estimate of the robustness of randomized ensembles -- a point we address in the next section.

\section{Limitations of Adaptive PGD} \label{sec:pgd}
We analyze the nature of adversarial perturbations obtained via adaptive PGD (APGD) for the special case of an REC -- an ensemble of binary linear classifiers (BLCs). In doing so, we unveil the reason underlying APGD's weakness in attacking RECs. 

Assume an ensemble $(\calF, \bm{\alpha})$ of $M$ BLCs:
\begin{equation}
    f_i(\vc{x}) = \tp{\vc{w}}_i\vc{x} + b_i \ \ \ \forall i \in [M]
\end{equation}
where $\vc{w}_i \in \reals^D$ and $b_i \in \reals$ are the weight and bias, respectively, of the $i^{th}$ BLC. The BLC $f_i$ assigns label $1$ to $\vc{x}$ if and only if $f_i(\vc{x})>0$, and $-1$ otherwise \footnote{For notional convenience, we assume $y\in\{-1,1\}$ for binary classification.}. The associated REC is given by $\Pr\{g(\vc{x}) = f_d(\vc{x})\}= \alpha_d$ where $d\in[M]$. 
We direct the reader to Appendix~\ref{app:proof} for detailed derivations and proofs of the results in this section.
\subsection{Attacking the Auxiliary Classifier}
The APGD update rule employs the expected loss $ \mean{l(\vc{x}+\bm{\delta},y)}$ in \eqref{eq:pgd}, where $l$ is the binary cross-entropy loss. Now consider the gradient $\vc{g}$ of the expected loss. By linearity of the expectation, we have:
\begin{align} \label{eq:grad-eq}
\begin{split}
     \vc{g} &= -y \sum_{i=1}^M \alpha_i\lambda_i \vc{w}_i
\end{split}
\end{align}
where $\lambda_i$'s belong to the open interval $(0,1)$.
Examining the expression in \eqref{eq:grad-eq}, we observe that the direction of $\vc{g}$ is determined purely by a linear combination of the classifiers' weight vectors. With that mind, if we define an \textit{auxiliary} classifier $\bar{f}$ such that:
\begin{equation}
    \bar{f}(\vc{x}) = \tp{\left(\sum_{i=1}^M \alpha_i\lambda_i \vc{w}_i\right)}\vc{x} + \bar{b}
\end{equation}
for some arbitrary $\bar{b}\in \reals$, then it is easy to see that the APGD attack is actually equivalent to the standard PGD attack evaluated against the auxiliary classifier $\bar{f}$. More generally, we have the following result:
\begin{theorem}[Auxiliary Classifiers]
\label{thm:auxiliary} For any REC $(\calF,\bm{\alpha})$ consisting of BLCs and any data-point $(\vc{x},y)$, there exists a sequence of auxiliary BLCs $\{\bar{f}^{(k)}\}_{k=1}^K $ such that in the $k^{\text{th}}$ iteration, the output of APGD against the REC is equal to the output of PGD against $\bar{f}^{(k)}$.
\end{theorem}

\begin{figure}[!t]
  \centering
   \includegraphics[width=0.7\columnwidth]{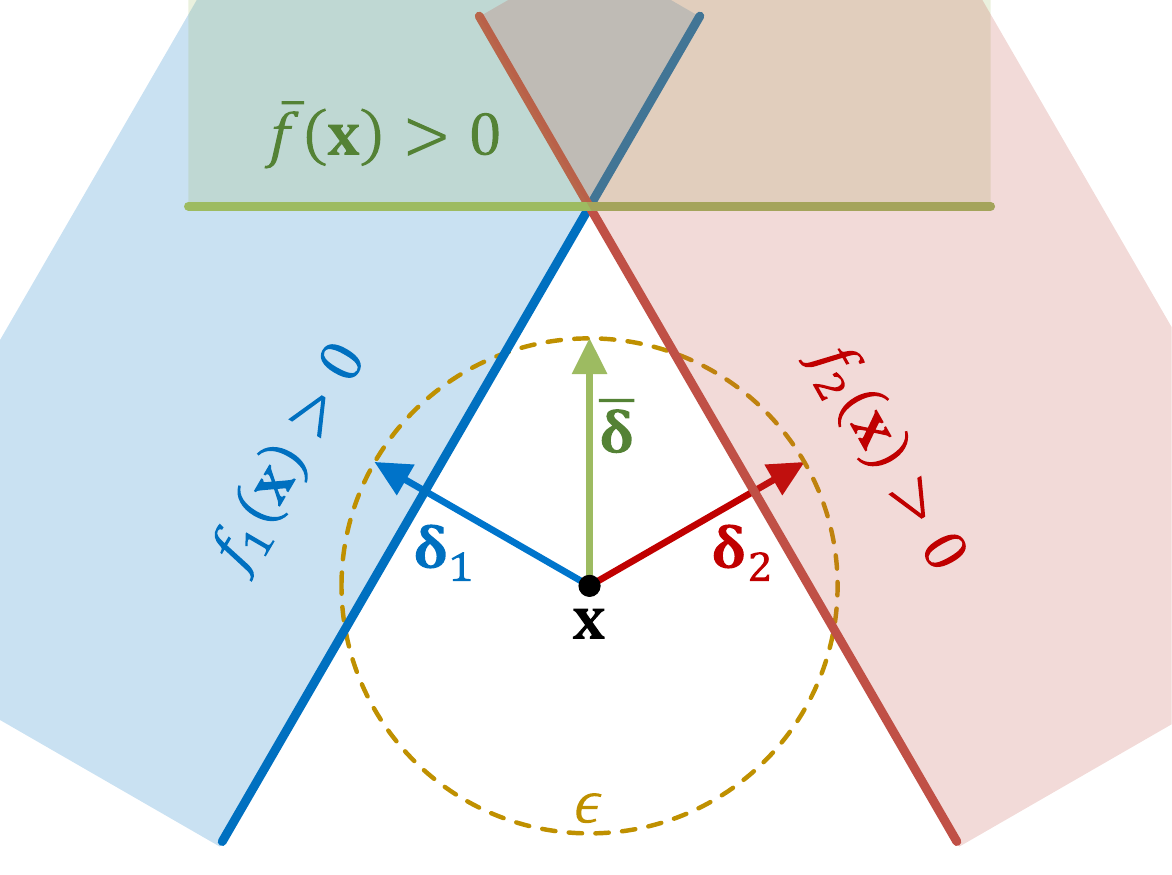}
  \caption{Illustration of an auxiliary classifier $\bar{f}$ for a randomized ensemble of two BLCs $f_1$ and $f_2$ with equiprobable sampling. The APGD output $\bar{\bm{\delta}}$ when attacking $f_1$ and $f_2$ is equivalent to the output of PGD when attacking $\bar{f}$.}\label{fig:aux}%
\end{figure}

\subsection{Inconsistency of Adaptive PGD}
Theorem~\ref{thm:auxiliary} highlights a fundamental flaw underlying APGD when applied to RECs: \emph{APGD iterates target an auxiliary classifier that may or may not exist in the REC being attacked}. For instance, Fig.~\ref{fig:aux} illustrates the case of an $M=2$ REC. Assuming equiprobable sampling, clearly choosing either of $\bm{\delta}_1$ or $\bm{\delta}_2$ will fool the ensemble half the time on average. However, due to the norm radius choice, the APGD generated perturbation cannot fool either, and thus will give an \emph{incorrect} and \emph{optimistic} estimate of the REC's robustness. 

Another implication of Theorem~\ref{thm:auxiliary} is that APGD might pre-maturely converge to a fixed-point, that is $\bm{\delta}^{(k)} = \bm{\delta}^{(k-1)}$, which can happen when the auxiliary classifier becomes ill-defined and reduces to $\bar{f}^{(k)}(\vc{x}) = \bar{b}$. Furthermore, these fundamental flaws in APGD can be exploited by ensemble training techniques to artificially boost the robustness of the randomized ensemble. In fact, prime examples of this phenomena include adapting standard ensembling techniques \cite{kariyappa2019improvingGAL, yang2021trs} that promote input gradient dissimilarity to the randomized setting, something we explore in Section~\ref{sec:results}.

One final implication of Theorem~\ref{thm:auxiliary} is the following result regarding APGD:
\begin{theorem}[Inconsistency]
\label{thm:inconsistency}
The APGD algorithm for randomized ensembles of classifiers is inconsistent.
\end{theorem}
Theorem~\ref{thm:inconsistency} provides further theoretical justification of the shortcomings of APGD when attacking randomized ensembles. The inconsistency of APGD leads to overly optimistic robustness evaluation in case of DNNs (see Fig.~\ref{fig:motivation}).

\section{The ARC Algorithm}\label{sec:arc}
To properly assess the robustness of randomized ensembles, and avoid the pitfalls of APGD as seen in Section~\ref{sec:pgd}, we present a novel attack algorithm for solving \eqref{eq:opt}, namely the Attacking Randomized ensembles of Classifiers (ARC) algorithm. 

\subsection{Binary Linear Classifiers}
Analogous to the APGD analysis in Section~\ref{sec:pgd}, we first present a simpler version of ARC (Algorithm~\ref{alg:arc-blc}) tailored for randomized ensembles of BLCs.
\begin{algorithm}[tb]
   \caption{The ARC Algorithm for BLCs}
   \label{alg:arc-blc}
\begin{algorithmic}[1]
   \STATE {\bfseries Input:} REC $(\calF,\bm{\alpha})$, labeled data-point $(\vc{x},y)$, norm $p$, and radius $\epsilon$.
   \STATE {\bfseries Output:} Adversarial perturbation $\bm{\delta}$ such that $\pnorm{\bm{\delta}}{p}\leq \epsilon$.
   \STATE Initialize $\bm{\delta} \leftarrow \vc{0}$, $v \leftarrow L(\vc{x},y,\bm{\alpha})$ , $q \leftarrow \frac{p}{p-1}$
   \STATE Define $\calI$ such that $\alpha_{i} \geq \alpha_{j}$ $\forall i,j\in \calI$ and $i\leq j$. 

   \FOR{$i \in \calI$}
   \CCOMMENT{optimal unit $\ell_p$ norm adversarial direction for $f_i$} 
   \STATE $\vc{g} \leftarrow -y \frac{|\vc{w}_i|^{q-1} \odot \text{sgn}(\vc{w}_i)}{\norm{\vc{w}_i}_q^{q-1}}$ 
   \CCOMMENT{shortest $\ell_p$ distance between $\vc{x}$ and $f_i$}
   \STATE $\zeta \leftarrow \frac{|f_i(\vc{x})|}{\pnorm{\vc{w}_i}{q}}$ 
   \IF{$\zeta\geq \epsilon$ $\lor$ $i=1$}
   \STATE $\beta \leftarrow \epsilon$
   \ELSE 
    \STATE $\beta \leftarrow \frac{\epsilon}{\epsilon-\zeta}\left|\frac{y\tp{\vc{w}_i}\bm{\delta}}{\pnorm{\vc{w}_i}{q}} + \zeta \right| + \rho$ 

   \ENDIF
   \STATE       $\hat{\bm{\delta}} \leftarrow \epsilon\frac{\bm{\delta} + \beta\vc{g}}{\pnorm{\bm{\delta} + \beta\vc{g}}{p}}$ \COMMENT{candidate $\hat{\bm{\delta}}$ such that $\pnorm{\hat{\bm{\delta}}}{p}=\epsilon$}
   \STATE       $\hat{v} \leftarrow L(\vc{x}+\hat{\bm{\delta}},y,\bm{\alpha})$
   \CCOMMENT{if robustness does not increase, update $\bm{\delta}$}
   \IF{$\hat{v} \leq v$}
   \STATE $\bm{\delta} \leftarrow \hat{\bm{\delta}}$, $v \leftarrow \hat{v}$
   \ENDIF
   \ENDFOR
\end{algorithmic}
\end{algorithm}

In this setting, the ARC algorithm iterates over all members of the ensemble \emph{once} in a greedy fashion. At iteration $i \in [M]$, a candidate perturbation $\hat{\bm{\delta}}$ is obtained by updating the perturbation $\bm{\delta}$ from the previous perturbation as follows:
\begin{equation} \label{eq:arc-update}
    \hat{\bm{\delta}} = \gamma\left(\bm{\delta}  + \beta \vc{g}\right)
\end{equation}
where $\gamma >0$ such that $\pnorm{\hat{\bm{\delta}}}{p}=\epsilon$, $\vc{g}$ is the unit $\ell_p$ norm optimal perturbation \cite{melachrinoudis1997analytical} for fooling classifier $f_i$, and $\beta$ is adaptively chosen based on lines (10)-(13) to guarantee that $\hat{\bm{\delta}}$ can fool $f_i$, irrespective of $\bm{\delta}$, as long as the shortest $\ell_p$ distance between $f_i$ and $\vc{x}$ is less than the norm radius $\epsilon$. Subsequently, the ARC algorithm updates its iterate $\bm{\delta}$ with $\hat{\bm{\delta}}$ as long as the average classification accuracy satisfies:
\begin{equation}
    L(\vc{x}+\hat{\bm{\delta}},y,\bm{\alpha}) \leq L(\vc{x}+\bm{\delta},y,\bm{\alpha})
\end{equation}
that is $\hat{\bm{\delta}}$ does not \textit{increase} the robust accuracy. Note that $\rho>0$ in line (13) is an arbitrarily small positive constant to avoid boundary conditions. In our implementation, we set $\rho = 0.05\epsilon$.

In contrast to APGD, the ARC algorithm is provably consistent as stated below:
\begin{theorem}[Consistency]
\label{thm:consistency}
The ARC algorithm for randomized ensembles of binary linear classifiers is consistent.
\end{theorem}

The implication of Theorem~\ref{thm:consistency} (see Appendix~\ref{app:proof} for proof) is that the perturbation generated via the ARC algorithm  is \textit{guaranteed} to be adversarial to $(\calF,\bm{\alpha})$, given that $\vc{x}$ is correctly classified by all members provided such a perturbation exists. Thus, we have constructed a theoretically-sound algorithm for attacking randomized ensembles of binary linear classifiers.

\subsection{Differentiable Multiclass Classifiers}
We now extend the ARC algorithm presented in Algorithm~\ref{alg:arc-blc} to the more general case of $C$-ary differentiable classifiers, such as DNNs. When dealing with a non-linear differentiable classifier $f$, we can \say{linearize} its behavior around an input $\vc{x}$ in order to provide \emph{estimates} of both the shortest $\ell_p$ distance to the decision boundary $\tilde{\zeta}$, and the corresponding unit $\ell_p$ norm direction $\tilde{\vc{g}}$. Using a first-order Taylor series expansion at $\vc{x}$ we approximate $f$:
\begin{equation}
   f(\vc{u}) \approx \tilde{f}(\vc{u}) = f(\vc{x}) + \tp{\nabla f(\vc{x})}(\vc{u}-\vc{x})  
\end{equation}
where $\nabla f(\vc{x}) = \left[\nabla\left[f(\vc{x})\right]_1| ...| \nabla\left[f(\vc{x})\right]_C\right] \in \reals^{D\times C}$ is the Jacobian matrix. Let $m \in [C]$ be the assigned label to $\vc{x}$ by $f$, we can construct the approximate decision region $\tilde{\calR}_m(\vc{x})$ around $\vc{x}$ as follows:

\begin{align}
\begin{split}
     \tilde{\calR}_m(\vc{x}) &=  \bigcap_{\substack{j=1 \\ j \neq m}}^C \left\{ \vc{u} \in \reals^D: \left[\tilde{f}(\vc{u})\right]_m > \left[\tilde{f}(\vc{u})\right]_j\right\} \\
     &=  \bigcap_{\substack{j=1 \\ j \neq m}}^C \left\{ \vc{u} \in \reals^D: \tp{\tilde{\vc{w}}_j}\left(\vc{u}-\vc{x}\right) + \tilde{h}_j> 0\right\}
\end{split}
\end{align}
where $\tilde{h}_j = \left[f(\vc{x})\right]_m - \left[f(\vc{x})\right]_j$ and $\tilde{\vc{w}}_j = \nabla\left[f(\vc{x})\right]_m - \nabla\left[f(\vc{x})\right]_j$  $\forall j \in [C]\setminus\{m\}$. The decision boundary of $\tilde{\calR}_m(\vc{x})$ is captured via the $C-1$ hyper-planes $\calH_j$ defined by:
\begin{equation}
\calH_j = \left\{\vc{u}\in \reals^D: \tp{\tilde{\vc{w}}_j}\left(\vc{u}-\vc{x}\right) + \tilde{h}_j= 0 \right\}
\end{equation}
Therefore, in order to obtain $\tilde{\zeta}$ and $\tilde{\vc{g}}$, we find the closest hyper-plane $\calH_n$:
\begin{equation}\label{eq:arc-search}
    n= \argmin_{{j\in[C]\setminus \{m\}}} \frac{\left|\tilde{h}_j\right|}{\pnorm{\tilde{\vc{w}}_j}{q}}  = \argmin_{{j\in[C]\setminus \{m\}}} \tilde{\zeta}_j
\end{equation}
and then compute:
\begin{equation}
 \tilde{\zeta} = \tilde{\zeta}_n \ \ \ \ \& \ \ \ \  \tilde{\vc{g}} = -\frac{|\tilde{\vc{w}}_n|^{q-1} \odot \text{sgn}(\tilde{\vc{w}}_n)}{\norm{\tilde{\vc{w}}_n}_q^{q-1}}  
\end{equation}
where $q\geq 1$ is the dual norm of $p$, $\odot$ is the Hadamard product, and $\vc{m} = |\vc{w}|^r$ denotes the element-wise operation $m_i = |w_i|^r$ $\forall i\in[D]$. Using these equations, we generalize Algorithm~\ref{alg:arc-blc} for BLCs to obtain ARC in Algorithm~\ref{alg:arc}. The ARC algorithm is iterative. At each iteration $k \in [K]$, ARC performs a local linearization approximation. Due to this approximation, we limit the local perturbation radius to $\eta$, a hyper-paramater often referred to as the step size akin to PGD.

\subsection{Practical Considerations} \label{ssec:efficient}
The ARC algorithm is extremely efficient to implement in practice. Automatic differentiation packages such as PyTorch \cite{paszke2017automatic} can carry out the linearization procedures in a seamless fashion. However, datasets with large number of classes $C$, e.g., CIFAR-100 and ImageNet, can pose a practical challenge for Algorithm~\ref{alg:arc}. This is due to the $\mathcal{O}(C)$ search performed in \eqref{eq:arc-search}, which recurs for each iteration $k$ and classifier $f_i$ in the ensemble. However, we have observed that \eqref{eq:arc-search} can be efficiently and accurately approximated by limiting the search only to a fixed subset of hyper-planes of size $G \in [C-1]$, e.g., $G=4$ reduces the evaluation time by more than $\sim 14\times$ (for CIFAR-100) without compromising on accuracy. We shall use this version of ARC for evaluating such datasets, and refer the reader to Appendix~\ref{app:arclite} for further details.
\begin{algorithm}[!t]
   \caption{The ARC Algorithm}
   
   \label{alg:arc}
\begin{algorithmic}[1]
   \STATE {\bfseries Input:} REC $(\calF,\bm{\alpha})$, labeled data-point $(\vc{x},y)$, number of steps $K$, step size $\eta$, norm $p$, and radius $\epsilon$.
   \STATE {\bfseries Output:} Adversarial perturbation $\bm{\delta}$ such that $\pnorm{\bm{\delta}}{p}\leq \epsilon$.
   \STATE Initialize $\bm{\delta} \leftarrow \vc{0}$, $v \leftarrow L(\vc{x},y,\bm{\alpha})$ , $q \leftarrow \frac{p}{p-1}$
   \STATE Define $\calI$ such that $\alpha_{i} \geq \alpha_{j}$ $\forall i,j\in \calI$ and $i\leq j$. 
   \FOR{$k \in [K]$}
   \STATE $\bm{\delta}_l \leftarrow \vc{0}$, $v_l \leftarrow v$ \COMMENT{initialize for local search}
   \FOR{$i \in \calI$}
   \STATE $\tilde{\vc{x}} \leftarrow \vc{x} +\bm{\delta}+\bm{\delta}_l $
    
    \CCOMMENT{the label assigned to $\tilde{\vc{x}}$ by $f_i$}  
   \STATE $m\leftarrow \argmax_{j\in [C]} \left[f_i(\tilde{\vc{x}})\right]_j$

   \STATE $\tilde{\vc{w}}_j \leftarrow \nabla\left[f_i(\tilde{\vc{x}})\right]_m - \nabla\left[f_i(\tilde{\vc{x}})\right]_j$ \hfill$\forall j\in[C]\setminus\{m\}$
   \STATE $\tilde{h}_j \leftarrow \left[f_i(\tilde{\vc{x}})\right]_m - \left[f_i(\tilde{\vc{x}})\right]_j$\hfill$\forall j\in[C]\setminus\{m\}$
    \CCOMMENT{get closest hyper-plane to $\tilde{\vc{x}}$ when $f_i$ is linearized}   
   \STATE  $n\leftarrow \argmin_{{j\in[C]\setminus \{m\}}} \frac{\left|\tilde{h}_j\right|}{\pnorm{\tilde{\vc{w}}_j}{q}}$
   
   \STATE $\vc{g} \leftarrow -\frac{|\tilde{\vc{w}}_n|^{q-1} \odot \text{sgn}(\tilde{\vc{w}}_n)}{\norm{\tilde{\vc{w}}_n}_q^{q-1}}$ 
   \STATE $\zeta \leftarrow \frac{\left|\tilde{h}_n\right|}{\pnorm{\tilde{\vc{w}}_n}{q}}$
   \IF{$\zeta \geq \eta \lor i=1$}
        \STATE $\beta \leftarrow \eta$
    \ELSE
        \STATE $\beta \leftarrow \frac{\eta}{\eta-\zeta}\left|\frac{\tp{\tilde{\vc{w}}_n}\bm{\delta}_l}{\pnorm{\tilde{\vc{w}}_n}{q}} + \zeta \right| + \rho$
    \ENDIF
   \STATE       $\hat{\bm{\delta}}_l \leftarrow \eta\frac{\bm{\delta}_l + \beta\vc{g}}{\pnorm{\bm{\delta}_l + \beta\vc{g}}{p}}$ 
   \CCOMMENT{project onto the $\ell_p$ ball of radius $\epsilon$ and center $\vc{x}$}
   \STATE $\hat{\bm{\delta}} \leftarrow \Pi_\epsilon^p \left( \bm{\delta} + \hat{\bm{\delta}}_l\right)$, $\hat{v}_l \leftarrow L(\vc{x}+\hat{\bm{\delta}},y,\bm{\alpha})$
   \IF{$\hat{v}_l \leq v_l$}
   \STATE $\bm{\delta}_l \leftarrow \hat{\bm{\delta}}_l$, $v_l \leftarrow \hat{v}_l$ \COMMENT{update the local variables}
   \ENDIF
   \ENDFOR
   \CCOMMENT{update the global variables after localized search}
   \STATE         $\hat{\bm{\delta}} \leftarrow \Pi_\epsilon^p\left(\bm{\delta} + \bm{\delta}_l\right)$, $\hat{v} \leftarrow L(\vc{x}+\hat{\bm{\delta}},y,\bm{\alpha})$
   \IF{$\hat{v} \leq v$}
   \STATE $\bm{\delta} \leftarrow \hat{\bm{\delta}}$, $v \leftarrow \hat{v}$
   \ENDIF
   \ENDFOR
   
\end{algorithmic}
\end{algorithm}

\section{Experiments} \label{sec:results}
\subsection{Setup}
We conduct comprehensive experiments to demonstrate the effectiveness of ARC in generating norm-bounded adversarial perturbations for RECs, as compared to APGD. Specifically, we experiment with a variety of network architectures (VGG-16 \cite{simonyan2014very}, ResNet-20, ResNet-18 \cite{he2016deep}, WideResNet-28-4 \cite{zagoruyko2016wide}, and MobileNetV1 \cite{howard2017mobilenets}), datasets (SVHN \cite{svhn}, CIFAR-10 \cite{cifar10}, CIFAR-100, and ImageNet \cite{krizhevsky2012imagenet}), and norms ($\ell_2$ and $\ell_\infty$). We use the tried and tested publicly available implementation of PGD from \cite{rice2020overfitting}. For all our experiments, the same $\ell_p$ norm is used during both training and evaluation. Further details on the training/evaluation setup can be found in Appendix~\ref{app:setup}.

For BAT \cite{pinot2020randomization} comparisons, we first train a standard AT model using PGD \cite{rice2020overfitting}, which serves both as an independent robust baseline as well as the first model ($f_1$) in the REC to be constructed. Following the BAT procedure, we then train a second model ($f_2$) using the adversarial samples of the \emph{first} model only. The same network architecture will be used for both models. When evaluating the robust accuracy of an REC, we compute the true expectation via \eqref{eq:acc} in accordance with \cite{pinot2020randomization}.


\subsection{Results}

\begin{table}[!t]
\centering
\caption{Comparison between ARC and adaptive PGD when attacking randomized ensembles trained via BAT \cite{pinot2020randomization} across various network architectures and norms on the CIFAR-10 dataset. We use the standard radii $\epsilon_2 = 128/255$ and $\epsilon_\infty = 8/255$ for $\ell_2$ and $\ell_\infty$-bounded perturbations, respectively.} \label{tab:networks}
\vskip 0.15in
\begin{sc}
\resizebox{1\columnwidth}{!}{%
\begin{tabular}{l c   c c c  r }

\toprule
\multirow{3}{*}{Network}   & \multirow{3}{*}{Norm}   & \multicolumn{4}{c}{Robust Accuracy [\%]} \\
    & & AT ($M=1$) & \multicolumn{3}{c}{REC ($M=2$)} \\
     & & PGD & APGD & ARC & Diff\\
\midrule
\multirow{2}{*}{ResNet-20}   & $\ell_2$    & $62.43$ & $69.21$  & $55.44$ & \bfu{$-13.77$}\\
&  $\ell_\infty$    & $45.66$ &$61.10$  & $40.71$ & \bfu{$-20.39$} \\
\midrule
\multirow{2}{*}{MobileNetV1}   & $\ell_2$    & $66.39$ & $67.92$  & $59.43$ & \bfu{$-8.49$}\\
&  $\ell_\infty$    & $49.23$ & $59.27$   & $44.59$ & \bfu{$-14.68$} \\
\midrule 
\multirow{2}{*}{VGG-16}   & $\ell_2$     & $66.08$ & $66.96$  & $59.20$ & \bfu{$-7.76$}\\
&  $\ell_\infty$    & $49.02$ & $57.82$  & $42.93$ & \bfu{$-14.89$} \\
\midrule
\multirow{2}{*}{ResNet-18}  &$\ell_2$   & $69.16$ & $70.16$  & $65.88$ & \bfu{$-4.28$}\\
&  $\ell_\infty$   & $51.73$ & $61.61$  & $47.43$ & \bfu{$-14.18$} \\
\midrule

\multirow{2}{*}{WideResNet-28-4}   & $\ell_2$     & $69.91$ & $71.48$  & $62.95$ & \bfu{$-8.53$}\\
&  $\ell_\infty$     & $51.88$ & $63.86$  & $48.65$ & \bfu{$-15.21$} \\
\bottomrule

\end{tabular}}
\end{sc}
\end{table}

We first showcase the 
effectiveness of ARC on CIFAR-10 using five network architectures and both $\ell_2$ and $\ell_\infty$ norms. Table~\ref{tab:networks} summarizes the robust accuracies for various models and adversaries. Following the approach of \cite{pinot2020randomization}, we find the optimal sampling probability $\bm{\alpha} = (\alpha, 1-\alpha)$ that maximizes the robust accuracy via a grid search using APGD. We report the corresponding robust accuracies using both APGD and ARC for the same optimal sampling probability, to ensure a fair comparison. We consistently find that $\alpha^*\approx 0.9$, that is the REC heavily favors the robust model ($f_1$). 

Table~\ref{tab:networks} provides \emph{overwhelming} evidence that: 1) BAT trained RECs perform significantly better than their respective AT counterparts, when APGD is used for robustness evaluation, 2) employing ARC instead of APGD results in a massive drop in robust accuracy of the \emph{same} REC, and 3) the \emph{true robust accuracy of the REC} (obtained via ARC) \emph{is worse than that of the AT baseline $M=1$}. For example, a randomized ensemble of ResNet-20s achieves $61.1\%$ robust accuracy when using APGD, a $\sim 15\%$ increase over the AT baseline. However, when using ARC, we find that the robustness of the ensemble is in fact $40.7\%$ which is: 1) much lower ($\sim 20\%$) than what APGD claims and 2) also lower ($\sim 5\%$) than the AT baseline.

The conclusions of Table~\ref{tab:networks} are further corroborated by Table~\ref{tab:datasets}, where we now compare ARC and APGD across four datasets. These empirical results illustrate both the validity and usefulness of our theoretical results, as they convincingly show that APGD is indeed ill-suited for evaluating the robustness of RECs and can provide a \emph{false} sense of robustness.

\begin{table}[!t]
\centering
\caption{Comparison between ARC and adaptive PGD when attacking randomized ensembles trained via BAT \cite{pinot2020randomization} across various datasets and norms. We use ResNet-18 for ImageNet and and ResNet-20 for SVHN, CIFAR-10, and CIFAR-100 datasets.} \label{tab:datasets}
\vskip 0.15in
\begin{sc}
\resizebox{1\columnwidth}{!}{%
\begin{tabular}{l c r  c c c  r}

\toprule
\multirow{3}{*}{Dataset}  & \multirow{3}{*}{Norm}  & \multirow{3}{*}{Radius ($\epsilon$) }  & \multicolumn{4}{c}{Robust Accuracy [\%]} \\
 &   & & AT ($M=1$) & \multicolumn{3}{c}{REC ($M=2$)} \\
  &   & & PGD & APGD & ARC & Diff\\
\midrule
\multirow{2}{*}{SVHN}   &  $\ell_2$ & $128/255$   & $68.35$ & $74.66$  & $60.15$ & \bfu{$-14.51$}\\
& $\ell_\infty$ & $8/255$    & $53.55$ & $65.99$  & $52.01$ & \bfu{$-13.98$} \\
\midrule
\multirow{2}{*}{CIFAR-10}   &  $\ell_2$ & $128/255$    & $62.43$ & $69.21$  & $55.44$ & \bfu{$-13.77$}\\
& $\ell_\infty$ & $8/255$   & $45.66$ & $61.10$  & $40.71$ & \bfu{$-20.39$} \\
\midrule
\multirow{2}{*}{CIFAR-100}   &  $\ell_2$ & $128/255$    & $34.60$ & $41.91$  & $28.92$ & \bfu{$-12.99$}\\
& $\ell_\infty$ & $8/255$    & $22.29$ & $33.37$  & $17.45$ & \bfu{$-15.92$} \\
\midrule
 \multirow{2}{*}{ImageNet} &  $\ell_2$ & $128/255$    & $47.61$ & $49.62$  & $42.09$ & \bfu{$-7.53$}  \\
 &  $\ell_\infty$ & $4/255$    & $24.33$ & $35.92$  & $19.54$ & \bfu{$-16.38$} \\
\bottomrule

\end{tabular}}
\end{sc}
\end{table}

\begin{figure*}[t]
  \centering
    \subfloat[]{\includegraphics[height=4cm]{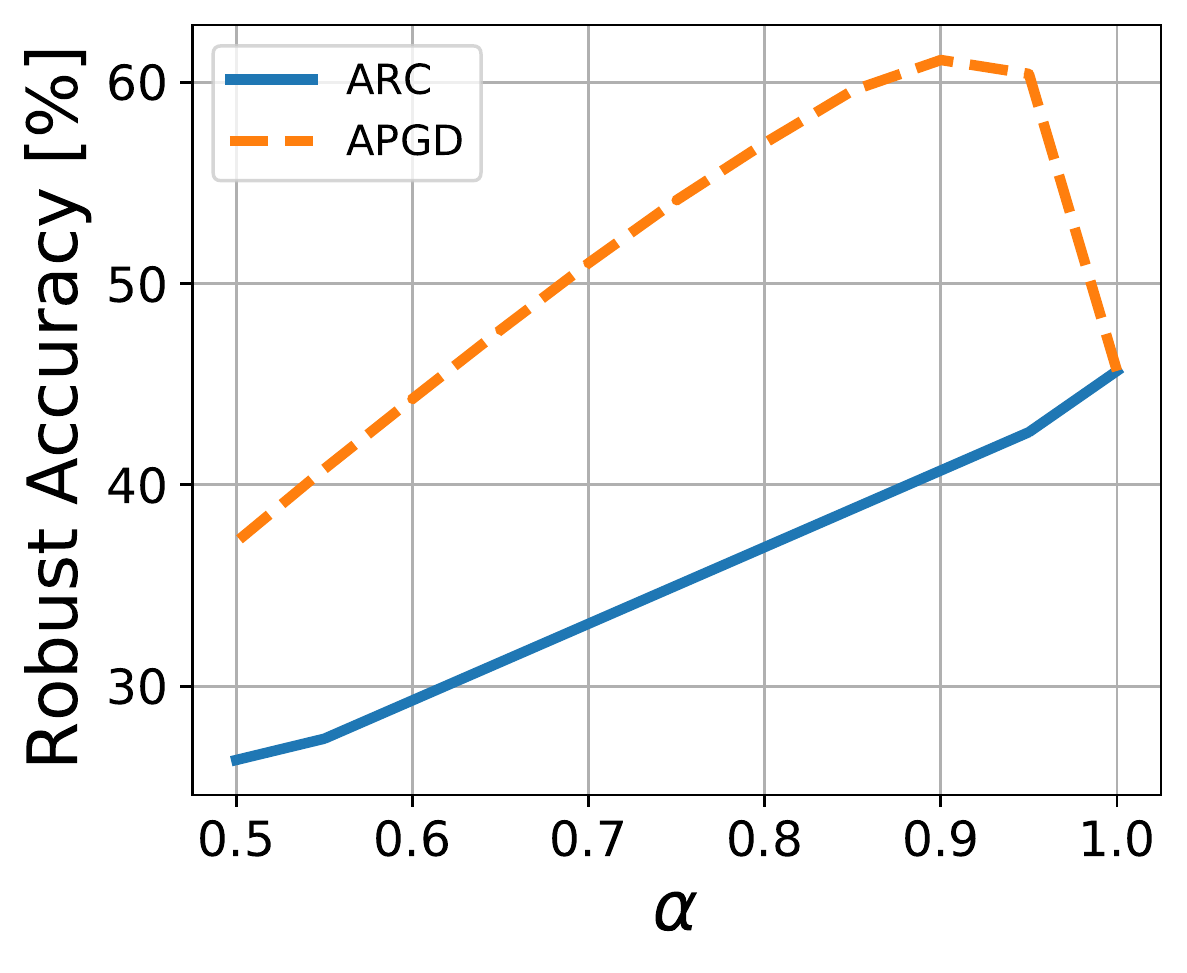}\label{fig:comp-alpha}}%
    \qquad%
    \subfloat[]{\includegraphics[height=4cm]{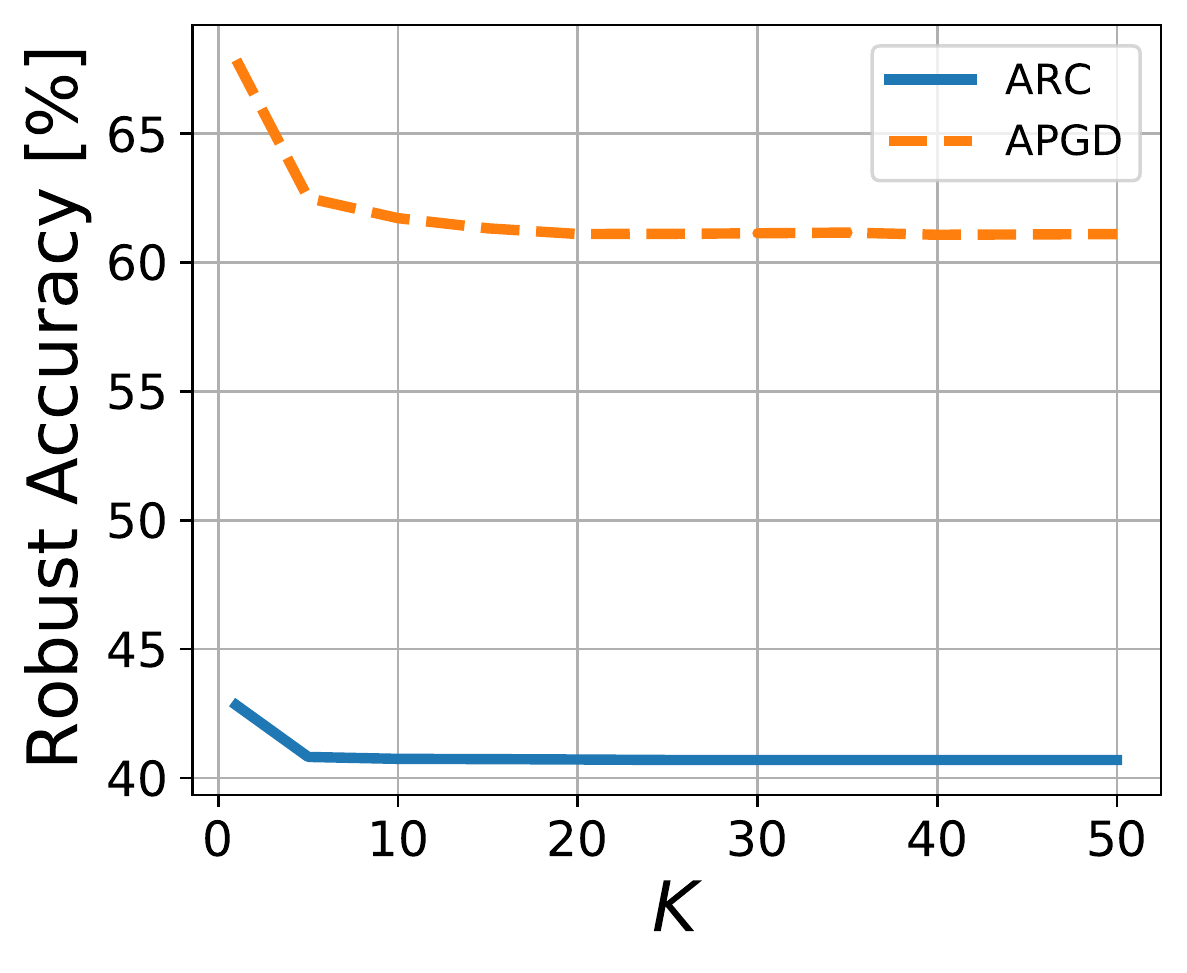}\label{fig:comp-k}}%
    \qquad%
    \subfloat[]{\includegraphics[height=4cm]{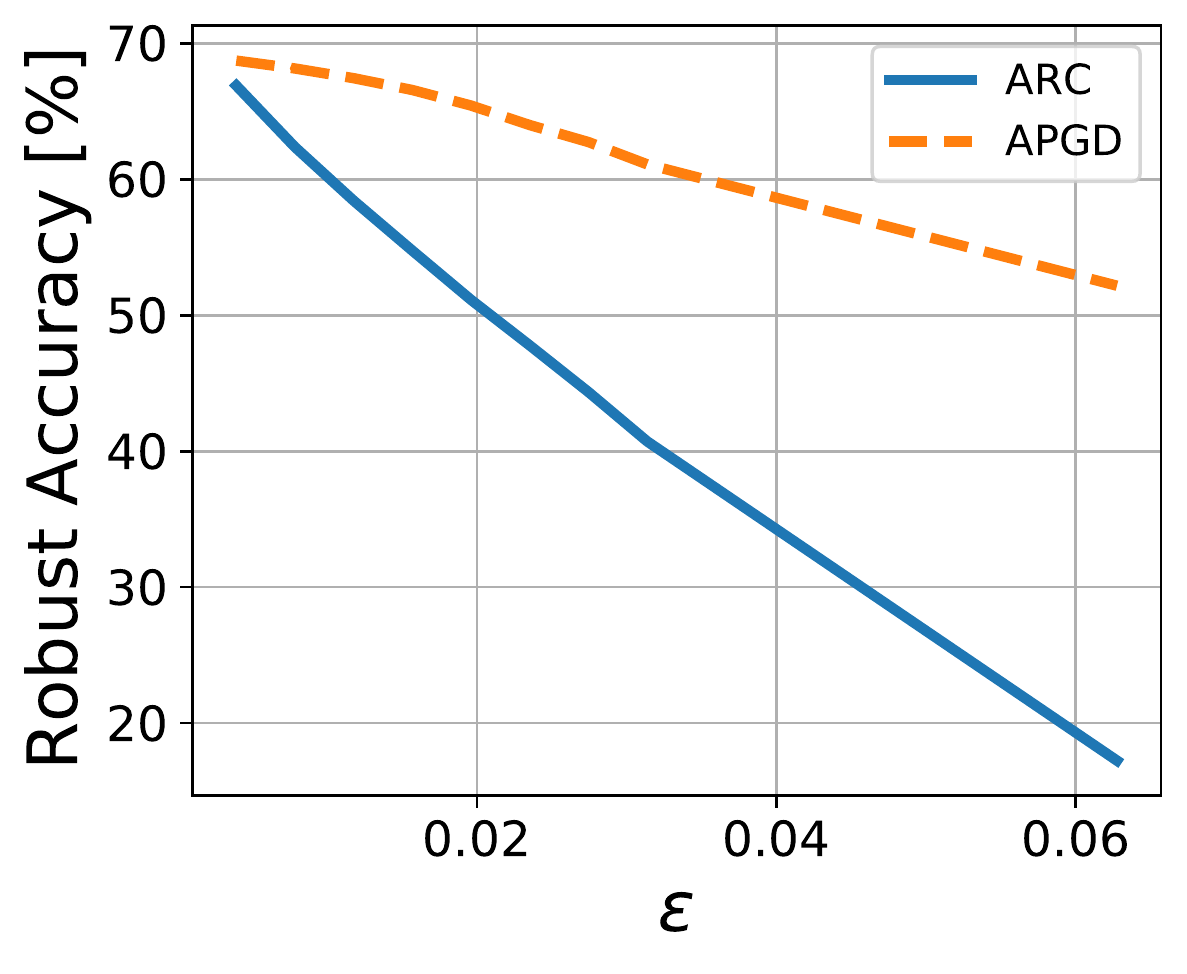}\label{fig:comp-epsilon}}
  \caption{Comparison between ARC and APGD using an REC of ResNet-20s obtained via $\ell_\infty$ BAT ($M=2$) on CIFAR-10. Robust accuracy vs.: (a) sampling probability $\alpha$, (b) number of iterations $K$, and (c) norm radius $\epsilon$.}
  \label{fig:comp}
\end{figure*}

\subsection{Robustness Stress Tests} \label{ssec:tests}
In this section, we conduct robustness stress tests to further ensure that our claims are in fact well-founded and not due to a technicality.   

\textbf{Parameter Sweeps}: We sweep the sampling probability (Fig.~\ref{fig:comp-alpha}), number of iterations (Fig.~\ref{fig:comp-k}), and norm radius (Fig.~\ref{fig:comp-epsilon}) for an REC of two ResNet-20s obtained via $\ell_\infty$ BAT on the CIFAR-10 dataset. For similar experiments across more norms and datasets, please refer to Appendix~\ref{app:sweeps}. In all cases, we find that the ARC adversary is consistently a much stronger attack for RECs than APGD. Moreover, we observe that $\alpha=1$ always results in the highest ensemble robustness when properly evaluated with ARC. This renders BAT RECs obsolete since the best policy is to choose the deterministic AT trained model $f_1$.

\textbf{Stronger PGD}: We also experiment with stronger versions of PGD using random initialization (R) \cite{madry2018towards} and 10 random restarts (RR) \cite{maini2020adversarial}. Table~\ref{tab:stress} reports the corresponding robust accuracies. Interestingly enough, we find that APGD+RR is actually slightly less effective than vanilla APGD and APGD+R. While this observation might seem counter-intuitive at first, it can be reconciled using our analysis from Section~\ref{sec:pgd}. By design, we have that the expected loss associated with APGD+RR perturbations is guaranteed to be larger than that of APGD+R\footnote{Assuming they share the same random seed.}. However, we have established (see Fig.~\ref{fig:aux}) that the perturbation that maximizes the expected loss can perform worse than other sub-optimal perturbations, which provides further evidence that APGD is indeed fundamentally flawed, and that ARC is better suited for evaluating the robustness of RECs.

\begin{table}[t]
\centering
\caption{Robust accuracy of different versions of APGD compared to ARC for ResNet-20 RECs on CIFAR-10. We use the standard radii $\epsilon_2 = 128/255$ and $\epsilon_\infty = 8/255$ for $\ell_2$ and $\ell_\infty$-bounded perturbations, respectively.}
\label{tab:stress}
\vskip 0.1in
\begin{sc}
\resizebox{0.8\columnwidth}{!}{%
\begin{tabular}{l c c }

\toprule
\multirow{2}{*}{Attack Method}  & \multicolumn{2}{c}{Robust Accuracy [\%]}\\
& $\ell_2$& $\ell_\infty$   \\
\midrule

APGD & $69.21$ &  $61.10$ \\
APGD + R & $69.28$ &  $62.80$ \\
APGD + RR& $70.14$ & $62.95$  \\
\midrule
ARC & $55.44$ &  $40.71$ \\

\bottomrule

\end{tabular}
}
\end{sc}
\vspace{-1em}
\end{table}

\subsection{Other Ensemble Training Methods} \label{ssec:ensemble}
As alluded to in Section~\ref{sec:background}, one can construct randomized ensembles from any pre-trained set of classifiers, regardless of the training procedure. To that end, we leverage the publicly available DVERGE \cite{yang2020dverge} trained ResNet-20 models on CIFAR-10 to construct and evaluate such RECs. 

Figure~\ref{fig:dverge-cross} plots the cross-robustness matrix of three DVERGE models. As expected, each model is quite robust ($\sim 87\%$) to other models' adversarial perturbations (obtained via PGD). However, the models are completely compromised when subjected to their own adversarial perturbations. 

\begin{figure}[!t]
  \centering
    \subfloat[]{\includegraphics[height=3.5cm]{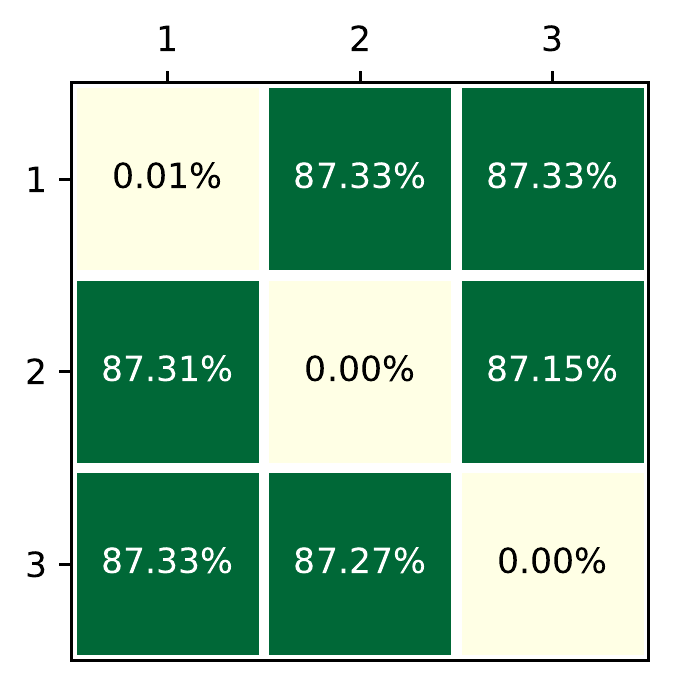}\label{fig:dverge-cross}}%
    \qquad%
    \subfloat[]{\includegraphics[height=3.5cm]{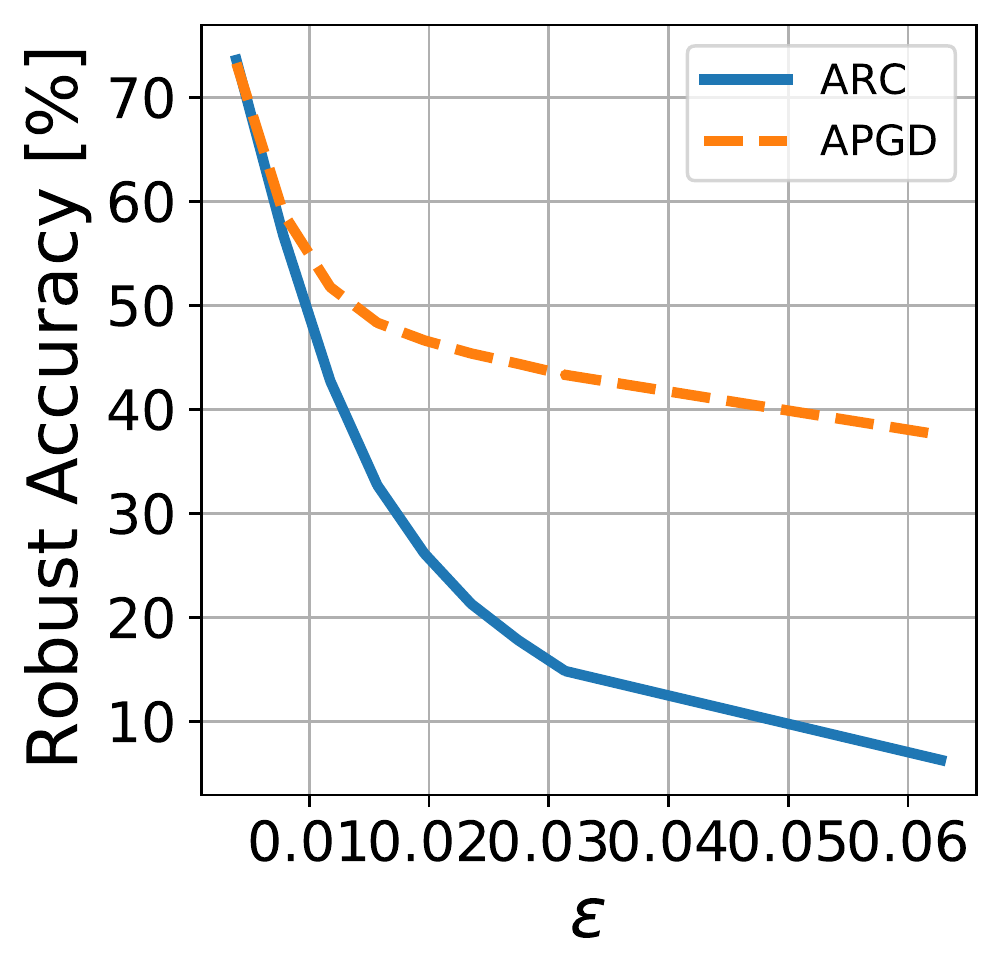}\label{fig:dverge-sweep}}%
 \label{fig:dverge}
  \caption{Robustness performance of an REC constructed from $\ell_\infty$ DVERGE trained models ($M=3$): (a) cross-robustness matrix, and (b) robust accuracy vs. norm radius $\epsilon$ using both ARC and APGD.}
\end{figure}

Due to the apparent symmetry between the models, we construct a randomized ensemble with equiprobable sampling $\alpha_i = 1/3$. We evaluate the robustness of the REC against both APGD and ARC adversaries, and sweep the norm radius $\epsilon$ in the process. As seen in Fig.~\ref{fig:dverge-sweep}, the ARC adversary is much more effective at fooling the REC, compared to APGD. For instance, for the typical radius $\epsilon=8/255$, the APGD robust accuracy is $43.36\%$, on-par with the single model AT baseline (see Table~\ref{tab:networks}). In contrast, the ARC robust accuracy is $14.88\%$, yielding a \emph{massive} difference in robustness ($\sim 28\%$). In the interest of space, we defer the ADP \cite{pang2019improvingADP} and TRS \cite{yang2021trs} experiments to Appendix~\ref{app:ensemble}. We also explore constructing RECs using independent adversarially trained models in Appendix~\ref{app:iat}.

\section{Conclusion}
We have demonstrated both theoretically and empirically that the proposed ARC algorithm is better suited for evaluating the robustness of randomized ensembles of classifiers. In contrast, current adaptive attacks provide an overly optimistic estimate of robustness. Specifically, our work successfully \emph{compromises} existing empirical defenses such as BAT \cite{pinot2020randomization}. While we also experiment with other methods for constructing RECs, we point out the paucity of work on randomized ensemble defenses.

Despite the vulnerability of existing methods, the hardware efficiency of randomized ensembles is an attractive feature. We believe that constructing robust randomized ensembles is a promising research direction, and our work advocates the need for improved defense methods including \textit{certifiable} defenses, i.e., defenses that are \emph{provably} robust against a specific adversary.

\section*{Acknowledgements}

This work was supported by the Center for Brain-Inspired Computing (C-BRIC) and the Artificial Intelligence Hardware (AIHW) program funded by the Semiconductor Research Corporation (SRC) and the Defense Advanced Research Projects Agency (DARPA).


\bibliography{ref}
\bibliographystyle{icml2022}

\newpage
\appendix
\onecolumn

\section{Proofs}\label{app:proof}
In this section, we provide proofs for Theorems \ref{thm:auxiliary}, \ref{thm:inconsistency} \& \ref{thm:consistency} as well as any omitted derivations. We first state the following result due to \cite{melachrinoudis1997analytical} on the $\ell_p$ norm projection of any point onto a hyper-plane:
\begin{lemma}[Melachrinoudis] \label{lem:lp}
For any hyper-plane $\calH = \{\vc{u}\in \reals^D: \vc{w}^{{\text{\normalfont T}}}\vc{u}+b = 0\} \subseteq \reals^D$, point $\vc{x}\in\reals^D$, and $p\geq 1$, then the $\ell_p$ norm projection $\vc{z}^*$ of $\vc{x}$ onto $\calH$ can be obtained analytically as follows:
\begin{equation}
    \vc{z}^*  =  \vc{x} + \zeta \vc{g} = \vc{x}  -\zeta\text{\normalfont sgn}(f(\vc{x}))  \frac{|\vc{w}|^{q-1} \odot \text{\normalfont sgn}(\vc{w})}{\norm{\vc{w}}_q^{q-1}}
\end{equation}
where
\begin{equation}
    \zeta = \min_{\vc{z}\in \calH} \pnorm{\vc{z}-\vc{x}}{p} = \pnorm{\vc{z}^*-\vc{x}}{p} = \frac{\left|\vc{w}^{{\text{\normalfont T}}}\vc{x}+b\right|}{\pnorm{\vc{w}}{q}} = \frac{\left|f(\vc{x})\right|}{\pnorm{\vc{w}}{q}}
\end{equation}

is the shortest $\ell_p$ distance between $\vc{x}$ and $\calH$, $q \geq 1$ satisfying $\frac{1}{q} + \frac{1}{p}=1$ defines the dual norm of $p$, $\odot$ is the Hadamard product, $\vc{g}$ is the unit $\ell_p$ norm vector from $\vc{x}$ towards the decision boundary, and $\vc{m} = |\vc{w}|^r$ denotes the element-wise operation $m_i = |w_i|^r$ $\forall i\in[D]$.
\end{lemma}
 The only dependence $\vc{g}$ has on $\vc{x}$ is the sign of $f(\vc{x})$, which essentially determines which side of the hyper-plane $\vc{x}$ lies in. The nice thing about Lemma~\ref{lem:lp} is that it provides us with the tools for computing the optimal adversarial perturbations for linear classifiers, for all valid norms $p\geq 1$.
\subsection{Fooling a Binary Linear Classifier}
We also state and prove a simple and useful result for fooling binary linear classifiers. Let $(\vc{x},y) \in \reals^D\times \{-1,1\}$ be a labeled data point that is correctly classified by $f$. That is we have $y(\tp{\vc{w}}\vc{x}+b) >0$. 
Let $\bm{\delta} \in \reals^D$ such that $\tilde{\vc{x}} = \vc{x} + \bm{\delta}$ is misclassified by $f$. That is, $y(\tp{\vc{w}}\tilde{\vc{x}}+b) <0$. What can we say about $\bm{\delta}$? 
\begin{lemma}\label{lem:fool}
For any $(\vc{x},y) \in \reals^D\times\{-1,1\}$ correctly classified by $f$ and valid norm $p\geq 1$, $\bm{\delta}$ is adversarial to $f$ at $(\vc{x},y)$ if and only if the following conditions hold
\begin{equation}
    -y\vc{w}^{\text{\normalfont T}}\vc{\bm{\delta}}>0  \ \ \ \ \ \text{\normalfont \& }\ \ \ \    \frac{\left|\vc{w}^{\text{\normalfont T}}\vc{\bm{\delta}}\right|}{\pnorm{\vc{w}}{q}} > \zeta
\end{equation}
where 
\begin{equation}
    \zeta = \frac{|f(\vc{x})|}{\pnorm{\vc{w}}{q}} = y\frac{f(\vc{x})}{\pnorm{\vc{w}}{q}}
\end{equation}
is the shortest $\ell_p$ distance between $\vc{x}$ and the decision boundary of $f$.
\end{lemma}
\begin{proof}
We can write:
\begin{equation}
    y(\tp{\vc{w}}\tilde{\vc{x}}+b) = y(\tp{\vc{w}}\vc{x}+b) + y\tp{\vc{w}}\bm{\delta}
\end{equation}
Since $y(\tp{\vc{w}}\vc{x}+b)>0$ holds by assumption, $y(\tp{\vc{w}}\tilde{\vc{x}}+b)$ is negative if and only if $-y\tp{\vc{w}}\bm{\delta} > y(\tp{\vc{w}}\vc{x}+b)>0$. This implies we need $\bm{\delta}$ to have a positive projection on $-y\vc{w}$. Dividing both sides by $\pnorm{\vc{w}}{q}$, establishes the second condition as well.
\end{proof}

\subsection{Derivation of \eqref{eq:grad-eq}}

Let $(\calF, \bm{\alpha})$ be an REC of $M$ BLCs:
\begin{equation}
    f_i(\vc{x}) = \tp{\vc{w}}_i\vc{x} + b_i \ \ \ \forall i \in [M]
\end{equation}
where $\vc{w}_i \in \reals^D$ and $b_i \in \reals$ are the weight and bias, respectively, of the $i^{th}$ BLC. The BLC $f_i$ assigns label $1$ to $\vc{x}$ if and only if $f_i(\vc{x})>0$, and $-1$ otherwise. 

Define the binary cross-entropy loss $l$ for classifier $f$ evaluated at $(\vc{x},y)$:
\begin{align}
\begin{split}
     l(\vc{x},y) &= l\left(f(\vc{x}),y\right) =  -\log\left(\frac{e^{f(\vc{x})}}{1+e^{f(\vc{x})}}\right)\left(\frac{1+y}{2}\right) -\log\left(\frac{1}{1+e^{f(\vc{x})}}\right)\left(\frac{1-y}{2}\right) \\
     &= -\log\left(z\right)\tilde{y} -\log\left(1-z\right)(1-\tilde{y})
\end{split}
\end{align}
where $z = \frac{e^{f(\vc{x})}}{1+e^{f(\vc{x})}}$ is the empirical probability of assigning label $1$ to $\vc{x}$ by $f$ and $\tilde{y} = \frac{1+y}{2} \in \{0,1\}$ is the binarized version of the signed label $y \in \{-1,1\}$.

\textbf{Computing the gradient of $l$}:

Let us first compute:
\begin{align}
\begin{split}
     \nabla_{\vc{x}} \log\left(z\right) &= \nabla_{\vc{x}} \log\left(e^{f(\vc{x})}\right) - \nabla_{\vc{x}} \log\left(1 + e^{f(\vc{x})}\right)\\
     &=\nabla_{\vc{x}} f(\vc{x}) - \frac{1}{1 + e^{f(\vc{x})}}\nabla_{\vc{x}}  e^{f(\vc{x})} \\
     &= \vc{w} - \frac{e^{f(\vc{x})}}{1 + e^{f(\vc{x})}}\vc{w} = (1-z)\vc{w}
\end{split}
\end{align}
Similarly, we compute:
\begin{align}
\begin{split}
     \nabla_{\vc{x}} \log\left(1-z\right) &= \vc{0} -\nabla_{\vc{x}} \log\left(1 + e^{f(\vc{x})}\right)\\
     &=  - \frac{1}{1 + e^{f(\vc{x})}}\nabla_{\vc{x}}  e^{f(\vc{x})} \\
     &= - \frac{e^{f(\vc{x})}}{1 + e^{f(\vc{x})}}\vc{w} = -z\vc{w}
\end{split}
\end{align}
The gradient of $l(\vc{x},y)$ w.r.t. $\vc{x}$ can thus be computed:
\begin{align}
\begin{split}
     \nabla_{\vc{x}} l(\vc{x},y) 
     &= -\nabla_{\vc{x}}  \log\left(z\right)\tilde{y} -\nabla_{\vc{x}} \log\left(1-z\right)(1-\tilde{y}) \\
     &= -(1-z)\tilde{y} \vc{w}+ z(1-\tilde{y})\vc{w} \\
     &= (-(1-z)\tilde{y} + z(1-\tilde{y}))\vc{w} \\
     &= (-\tilde{y} + z\tilde{y} + z - z\tilde{y})\vc{w} \\
     &= (-\tilde{y} + z)\vc{w}
\end{split}
\end{align}

Note that $z\in (0,1)$ and $\tilde{y}\in\{0,1\}$, therefore $-\tilde{y}+z$ has the same sign as $-y$. If $y =1$, then $\tilde{y} = 1$, and $-1+z \in (-1,0)$ is negative. If $y=-1$, then $\tilde{y}= 0$, and $z\in(0,1)$ is positive.
Therefore, we can write:
\begin{align} \label{eq:xntropy-grad}
\begin{split}
      \nabla_{\vc{x}} l(\vc{x},y)  
     &= -y\lambda\vc{w}
\end{split}
\end{align}
where 
\begin{align}
\begin{split}
    \lambda &= |-(1-z)\tilde{y} + z(1-\tilde{y})| = (1-z)\tilde{y} + z(1-\tilde{y}) \\
    & =\frac{1}{2}\frac{1+y + (1-y)e^{f(\vc{x})}}{1+e^{f(\vc{x})}} \in (0,1)   
\end{split}
\end{align}
\textbf{Gradient of the expected loss}:

Using \eqref{eq:xntropy-grad}, we can compute the gradient $\vc{g}$ of the expected loss $ \mean{l(\vc{x}+\bm{\delta},y)}$ employed by APGD: 
\begin{align}
\begin{split}
 \vc{g} &= \nabla_{\vc{x}} \mean{l(\vc{x}+\bm{\delta} ,y)} = \nabla_{\vc{x}} \sum_{i=1}^{M}\left[\alpha_i l_i(\vc{x}+\bm{\delta} ,y)\right] \\
  &= \sum_{i=1}^{M}\left[\alpha_i \nabla_{\vc{x}}l_i(\vc{x}+\bm{\delta} ,y)\right] = \sum_{i=1}^{M}\left[\alpha_i (-y\lambda_i\vc{w}_i)\right] \\
  &=  -y\sum_{i=1}^{M}\alpha_i\lambda_i\vc{w}_i
\end{split}
\end{align}
where $\forall i \in [M]$:
\begin{equation}
    \lambda_i = \frac{1}{2}\frac{1+y + (1-y)e^{f_i(\vc{x}+\bm{\delta})}}{1+e^{f_i(\vc{x}+\bm{\delta})}} \in (0,1)
\end{equation}

\subsection{Proof of Theorem \ref{thm:auxiliary}}
We provide a more detailed proof for Theorem~\ref{thm:auxiliary}, restated below:
\begin{theorem2}[Restated]
For any REC $(\calF,\bm{\alpha})$ consisting of BLCs and any data-point $(\vc{x},y)$, there exists a sequence of auxiliary BLCs $\{\bar{f}^{(k)}\}_{k=1}^K $ such that in the $k^{\text{th}}$ iteration, the output of APGD against the REC is equal to the output of PGD against $\bar{f}^{(k)}$.
\end{theorem2}
\begin{proof}
Let $(\calF,\bm{\alpha})$ be an arbitrary REC of BLCs: $f_i(\vc{x}) = \tp{\vc{w}}_i\vc{x}+b_i$ and $(\vc{x},y) \in \reals^D \times \{-1,1\}$ be an arbitrary labeled data point. Recall the APGD update rule is:
\begin{align}\label{eq:apgd-aux}
\begin{split}
  \bm{\delta}^{(k)} &= \Pi_\epsilon^p \left( \bm{\delta}^{(k-1)} + \eta \mu_p\left(\nabla_{\vc{x}} \mean{l(\vc{x}+\bm{\delta}^{(k-1)} ,y)}\right)\right)  \\
  &= \Pi_\epsilon^p \left( \bm{\delta}^{(k-1)} + \eta \mu_p\left(\nabla_{\vc{x}} \sum_{i=1}^{M}\left[\alpha_i l_i(\vc{x}+\bm{\delta}^{(k-1)} ,y)\right]\right)\right) \\
  &= \Pi_\epsilon^p \left( \bm{\delta}^{(k-1)} + \eta \mu_p\left( \sum_{i=1}^{M}\left[\alpha_i \nabla_{\vc{x}}l_i(\vc{x}+\bm{\delta}^{(k-1)} ,y)\right]\right)\right) \\
  &= \Pi_\epsilon^p \left( \bm{\delta}^{(k-1)} + \eta \mu_p\left( \sum_{i=1}^{M}\left[\alpha_i \vc{g}_i^{(k-1)}\right]\right)\right)
\end{split}
\end{align}
where $l_i(\vc{x},y) =l(f_i(\vc{x}),y)$ is the binary cross-entropy loss evaluated at $f_i(\vc{x})$. We can compute the gradients $\vc{g}_i^{(k-1)}$ analytically $\forall i \in[M], \forall k\in [K]$:
\begin{equation}
    \vc{g}_i^{(k-1)} = -y \lambda_i^{(k-1)}  \vc{w}_i
\end{equation}
where 
\begin{equation}
    \lambda_i^{(k-1)} = \frac{1}{2}\frac{1+y + (1-y)e^{f_i(\vc{x}+\bm{\delta}^{(k-1)})}}{1+e^{f_i(\vc{x}+\bm{\delta}^{(k-1)})}} \in (0,1)
\end{equation}
We construct the $k^{\text{th}}$ auxiliary classifier as follows:
\begin{equation}
    \bar{f}^{(k)} = \tp{\bar{\vc{w}}_k}\vc{x}+\bar{b}_k = \tp{\left(\sum_{i=1}^M \alpha_i  \lambda_i^{(k-1)}  \vc{w}_i\right)}\vc{x}+\bar{b}_k
\end{equation}
for some arbitrary $\bar{b}_k\in \reals$. We now proceed to prove the result via induction on $k$, assuming both the APGD and PGD algorithms share the same random initialization $\bm{\delta}^{(0)}$. Let $\{\bar{\bm{\delta}}^{(k)}\}$ be the sequence of PGD outputs when attacking the auxiliary classifier  $\bar{f}^{(k)}$.
Assume that $\bar{\bm{\delta}}^{(k-1)} = \bm{\delta}^{(k-1)}$ for some $k$, i.e., the perturbation generated by APGD applied to the REC is equal to the perturbation generated by PGD applied to the auxiliary classifier. Thus, the $k^{\text{th}}$ PGD output will be:
\begin{align}\label{eq:pgd-aux}
\begin{split}
  \bar{\bm{\delta}}^{(k)}  &= \Pi_\epsilon^p \left( \bar{\bm{\delta}}^{(k-1)} + \eta \mu_p\left( \nabla_{\vc{x}}l\left(\bar{f}^{(k)}(\vc{x}+ \bar{\bm{\delta}}^{(k-1)}) ,y\right)\right)\right) \\
  &= \Pi_\epsilon^p \left(  \bar{\bm{\delta}}^{(k-1)} + \eta \mu_p\left( -y \bar{\lambda}^{(k)}\bar{\vc{w}}_k\right)\right) \\
  &= \Pi_\epsilon^p \left(  \bar{\bm{\delta}}^{(k-1)} + \eta \mu_p\left( -y\bar{\vc{w}}_k\right)\right) \\
  &= \Pi_\epsilon^p \left(  \bar{\bm{\delta}}^{(k-1)} + \eta \mu_p\left( -y\left[\sum_{i=1}^M \alpha_i  \lambda_i^{(k-1)}  \vc{w}_i\right]\right)\right) \\
  &= \Pi_\epsilon^p \left( \bm{\delta}^{(k-1)} + \eta \mu_p\left( \sum_{i=1}^{M}\left[\alpha_i \vc{g}_i^{(k-1)}\right]\right)\right) \\
  &= \bm{\delta}^{(k)}
\end{split}
\end{align}

which is equal to the APGD output in \eqref{eq:apgd-aux}. Note that in \eqref{eq:pgd-aux} we use the following property:
\begin{equation}
    \mu_p(\gamma \vc{u}) = \text{sgn}(\gamma)\mu_p(\vc{u}) \ \ \ \ \ \forall \vc{u} \in \reals^D,  \gamma \in \reals
\end{equation}
coupled with the fact that $\bar{\lambda}^{(k)} \in (0,1)$.

Therefore, we have established the inductive step $\forall k$. Since both algorithms start from the same random initialization, then we also establish the base case: $\bar{\bm{\delta}}^{(0)} = \bm{\delta}^{(0)}$.

\end{proof}

\subsection{Proof of Theorem \ref{thm:inconsistency}}
We provide a more detailed proof for Theorem~\ref{thm:inconsistency}, restated below:
\begin{theorem2}[Restated]
The APGD algorithm for randomized ensembles of classifiers is inconsistent.
\end{theorem2}
\begin{proof}
To prove inconsistency, we construct an example where the algorithm fails to generate an adversarial perturbation given one exists. Consider the following setup:
\begin{enumerate}
    \item two binary linear classifiers with $\vc{w}_1 = -\vc{w}_2 = \vc{w} \neq \vc{0}$, and $b_1 = b_2 = b>0$
    \item equiprobable sampling $\bm{\alpha} = (0.5,0.5)$
    \item $(\vc{x},y)=(\vc{0},1)$ with norm-bound $\epsilon = 2b\frac{\pnorm{\vc{w}}{p}}{\pnorm{\vc{w}}{2}^2}$
    
\end{enumerate}
It is easy to show that both BLCs correctly classify $\vc{x}$, since $f_1(\vc{x}) = f_2(\vc{x}) = b >0$, thus we have the condition $L(\vc{x},y,\bm{\alpha})=1$. We also have a trivial choice of $\bm{\delta} = \epsilon \vc{w}/\pnorm{\vc{w}}{p}$, which satisfies both necessary and sufficient conditions for fooling $f_2$ shown below:
\begin{equation}
    -\tp{\vc{w}}_2\bm{\delta} = \tp{\vc{w}}\bm{\delta} = \epsilon \frac{\tp{\vc{w}}\vc{w}}{\pnorm{\vc{w}}{p}} = \epsilon \frac{\pnorm{\vc{w}}{2}^2}{\pnorm{\vc{w}}{p}}>0\ \ \ \ \ \text{\normalfont \& }\ \ \ \  \left|\tp{\vc{w}}_2\bm{\delta}\right|/\pnorm{\vc{w}_2}{q} = \frac{\tp{\vc{w}}\bm{\delta}}{\pnorm{\vc{w}}{q}} = \epsilon\frac{\pnorm{\vc{w}}{2}^2}{\pnorm{\vc{w}}{q} \pnorm{\vc{w}}{p}} = \frac{2b}{\pnorm{\vc{w}}{q}} > \zeta_2 
\end{equation}
However, we observe that $\bm{\delta}$ does not fool $f_1$, since $-y(\tp{\vc{w}_1}\bm{\delta}) = -y(2b)<0$. Therefore, we have constructed a norm-bounded perturbation $\bm{\delta}$ such that $L(\vc{x}+\bm{\delta},y,\bm{\alpha}) = 0.5 < 1$. 

Next, we show that the APGD algorithm evaluated against this REC cannot generate a norm-bounded perturbation that is adversarial. Recall that, starting from some random $\bm{\delta}^{(0)}$ such that $\pnorm{\bm{\delta}^{(0)}}{p}\leq \epsilon$, the APGD algorithm performs the following update rule:
\begin{align}
\begin{split}
  \bm{\delta}^{(k)} &= \Pi_\epsilon^p \left( \bm{\delta}^{(k-1)} + \eta \mu_p\left(\nabla_{\vc{x}} \mean{l(\vc{x}+\bm{\delta}^{(k-1)} ,y)}\right)\right)  \\
  &= \Pi_\epsilon^p \left( \bm{\delta}^{(k-1)} + \eta \mu_p\left(\nabla_{\vc{x}} \sum_{i=1}^{2}\left[\alpha_i l_i(\vc{x}+\bm{\delta}^{(k-1)} ,y)\right]\right)\right) \\
  &= \Pi_\epsilon^p \left( \bm{\delta}^{(k-1)} + \eta \mu_p\left( \sum_{i=1}^{2}\left[\alpha_i \nabla_{\vc{x}}l_i(\vc{x}+\bm{\delta}^{(k-1)} ,y)\right]\right)\right) \\
  &= \Pi_\epsilon^p \left( \bm{\delta}^{(k-1)} + \eta \mu_p\left( \sum_{i=1}^{2}\left[\alpha_i \vc{g}_i^{(k-1)}\right]\right)\right)
\end{split}
\end{align}
where $l_i(\vc{x},y) =l(f_i(\vc{x}),y)$ is the binary cross-entropy loss evaluated for $f_i(\vc{x})$. We can compute the gradients $\vc{g}_i^{(k-1)}$ analytically $\forall i \in\{1,2\}, \forall k\in [K]$:
\begin{equation}
    \vc{g}_i^{(k-1)} = -y \lambda_i^{(k-1)}  \vc{w}_i 
\end{equation}
where 
\begin{equation}
    \lambda_i^{(k-1)} = \frac{1}{2}\frac{1+y + (1-y)e^{f_i(\vc{x}+\bm{\delta}^{(k-1)})}}{1+e^{f_i(\vc{x}+\bm{\delta}^{(k-1)})}} \in (0,1)
\end{equation}

Assume that $\bm{\delta}^{(0)}$ satisfies $\tp{\vc{w}}\bm{\delta}^{(0)} = 0$, therefore we have $f_1(\vc{x}+\bm{\delta}^{(0)}) = f_2(\vc{x}+\bm{\delta}^{(0)}) = b$. The expected gradient below is employed by APGD to obtain the next perturbation $\bm{\delta}^{(1)}$:
\begin{align}
\begin{split}
  \sum_{i=1}^{2}\left[\alpha_i \vc{g}_i^{(0)}\right] &= -y \sum_{i=1}^2\left[\alpha_i \lambda_i^{(0)}  \vc{w}_i\right]    =  -\frac{y}{2}\left( \lambda_1^{(0)}\vc{w} - \lambda_2^{(0)}\vc{w}\right) =-\frac{y}{2}\vc{w}\left( \lambda_1^{(0)} - \lambda_2^{(0)}\right) \\
  &= -\frac{y}{2}\vc{w}\left( \frac{1}{2}\frac{1+y + (1-y)e^{f_1(\vc{x}+\bm{\delta}^{(0)})}}{1+e^{f_1(\vc{x}+\bm{\delta}^{(0)})}} - \frac{1}{2}\frac{1+y + (1-y)e^{f_2(\vc{x}+\bm{\delta}^{(0)})}}{1+e^{f_2(\vc{x}+\bm{\delta}^{(0)})}}\right) \\
  &= -\frac{y}{2}\vc{w}\left( \frac{1}{2}\frac{1+y + (1-y)e^{b}}{1+e^{b}} - \frac{1}{2}\frac{1+y + (1-y)e^{b}}{1+e^{b}}\right) \\
  &= \vc{0}
\end{split}
\end{align}
This implies that the APGD algorithm is stuck in a fixed-point $\forall k$: $\bm{\delta}^{(k)} = \bm{\delta}^{(0)}$, which means the final perturbation obtained via PGD $\bm{\delta}_{\text{APGD}} = \bm{\delta}^{(0)}$ is a non-adversarial perturbation since $\tp{\vc{w}}\bm{\delta}^{(0)} = 0$. Hence, we have $L(\vc{x}+\bm{\delta}_{\text{APGD}},y,\bm{\alpha}) = L(\vc{x},y,\bm{\alpha}) = 1 $.
\end{proof}
\subsection{Proof of Theorem \ref{thm:consistency}}
We now can prove the main result, restated below:
\begin{theorem2}[Restated]
The ARC algorithm for randomized ensembles of binary linear classifiers is consistent.
\end{theorem2}
\begin{proof}
To prove consistency, we need to show that for all random BLCs $(\calF,\bm{\alpha})$, valid norms $p\geq 1$, norm-bounds $\epsilon$, and labeled data-points $(\vc{x},y)$ such that $L(\vc{x},y,\bm{\alpha})=1$, the ARC generated perturbation $\bm{\delta}_{\text{ARC}}$ will result in $L(\vc{x}+\bm{\delta}_{\text{ARC}},y,\bm{\alpha})<1$, given that $\exists \bm{\delta}$ that satisfies $\pnorm{\bm{\delta}}{p}\leq \epsilon$ and $L(\vc{x}+\bm{\delta},y,\bm{\alpha})<1$.

First, let us establish some properties of the ARC algorithm. Define $\bm{\delta}^{(i)}$ to be the perturbation vector $\bm{\delta}$ at the end of the $i^{\text{th}}$ iteration of the ARC algorithm. Similarly, define $v^{(i)}$ to be the cost function $L(\vc{x}+\bm{\delta}^{(i)},y,\bm{\alpha})$. The algorithm is initialized with $\bm{\delta}^{(0)} = \vc{0}$ and $v^{(0)} = L(\vc{x},y,\bm{\alpha})=1$. Subsequently, the algorithm terminates with $\bm{\delta}_\text{ARC} = \bm{\delta}^{(M)}$ with a final average classification $v^{(M)} = L(\vc{x}+\bm{\delta}_\text{ARC},y,\bm{\alpha})$. We note that the sequence $\{v^{(i)}\}$ is non-increasing by design, with $v^{(i)} \leq L(\vc{x}, y, \bm{\alpha})$. This implies that the algorithm will never \emph{increase} the mean classification, and if $\exists m \in [M]$ such that $v^{(m)}<v^{(m-1)}$, then the algorithm terminates with $v^{(M)}\leq v^{(m)}<1$. Therefore, to show consistency, it is enough to find one such $m\in[M]$.

Assume that $L(\vc{x},y,\bm{\alpha})=1$ and $\exists \bm{\delta}$ such that $\pnorm{\bm{\delta}}{p}\leq \epsilon$ and $L(\vc{x}+\bm{\delta},y,\bm{\alpha})<1$, then from Lemma~\ref{lem:fool} we know that $\exists m\in [M]$ such that:
\begin{equation} \label{eq:fool-cond}
    -y\vc{w}_m^{\text{\normalfont T}}\vc{\bm{\delta}}>0  \ \ \ \ \ \text{\normalfont \& }\ \ \ \    \frac{\left|\vc{w}_m^{\text{\normalfont T}}\vc{\bm{\delta}}\right|}{\pnorm{\vc{w}_m}{q}} > \zeta_m
\end{equation}
where 
\begin{equation}
    \zeta_m = \frac{|f_m(\vc{x})|}{\pnorm{\vc{w}_m}{q}} = y\frac{f_m(\vc{x})}{\pnorm{\vc{w}_m}{q}}
\end{equation}
Combining the second inequality from \eqref{eq:fool-cond} with Hölder's inequality, we get:
\begin{equation}
    \pnorm{\vc{w}_m}{q} \pnorm{\bm{\delta}}{p} \geq  \left|\vc{w}_m^{\text{T}}\vc{\bm{\delta}}\right| > \zeta_m \pnorm{\vc{w}_m}{q}
\end{equation}
which implies that $\zeta_m < \epsilon$ since $\pnorm{\bm{\delta}}{p} \leq \epsilon$. We now have two cases to consider:

\textbf{Case 1:}
If $m=1$, the candidate perturbation will be $\hat{\bm{\delta}}=\epsilon \vc{g}$, where
\begin{equation}
    \vc{g} = -y \frac{|\vc{w}_m|^{q-1} \odot \text{sgn}(\vc{w}_m)}{\norm{\vc{w}_m}_q^{q-1}}
\end{equation}
We will now show that $\hat{\bm{\delta}}$ satisfies both necessary and sufficient conditions to fool $f_m$. First, we compute the dot-product:
\begin{align}
\begin{split}
     -y\tp{\vc{w}_m}\hat{\bm{\delta}} &= -y\tp{\vc{w}_m}\left(\epsilon \vc{g}\right) = - y\epsilon\tp{\vc{w}_m}\vc{g} = \epsilon\pnorm{\vc{w}_m}{q} >0
\end{split}
\end{align}
which is strictly positive, and the last equality uses the following:
\begin{align}
    \begin{split}
        - y\tp{\vc{w}_m}\vc{g} &= \frac{(-y)^2}{\pnorm{\vc{w}_m}{q}^{q-1}} \sum_{i=1}^D |w_{m,i}|^{q-1} \text{sgn}(w_{m,i}) w_{m,i} \\
        &= \frac{1}{\pnorm{\vc{w}_m}{q}^{q-1}} \sum_{i=1}^D |w_{m,i}|^{q-1} |w_{m,i}| \\
        &= \frac{1}{\pnorm{\vc{w}_m}{q}^{q-1}} \sum_{i=1}^D |w_{m,i}|^{q} \\
        &= \frac{1}{\pnorm{\vc{w}_m}{q}^{q-1}} \pnorm{\vc{w}_m}{q}^{q} = \pnorm{\vc{w}_m}{q}
    \end{split}
\end{align}
Second, we compute the following quantity:
\begin{equation}
    \frac{\left|\tp{\vc{w}_m}\hat{\bm{\delta}}\right|}{\pnorm{\vc{w}_m}{q}} = \frac{-y\tp{\vc{w}_m}\hat{\bm{\delta}}}{\pnorm{\vc{w}_m}{q}} = \frac{\epsilon\pnorm{\vc{w}_m}{q}}{\pnorm{\vc{w}_m}{q}} = \epsilon > \zeta_m
\end{equation}
which is strictly greater than $\zeta_m$. Therefore, the candidate $\hat{\bm{\delta}}$ satisfies both necessary and sufficient conditions for fooling $f_m$.

\textbf{Case 2:} If $m>1$, then assume that up until the $(m-1)^{\text{th}}$ iteration the ARC algorithm has not found any adversarial perturbation, that is $v^{(m-1)}=1$, otherwise the algorithm will terminate with an adversarial perturbation. Let $\vc{u} = \bm{\delta}^{(m-1)}$, and by assumption we have $L(\vc{x}+\vc{u},y,\bm{\alpha})=1$. During the $m^{\text{th}}$ iteration, we have the candidate perturbation $\hat{\bm{\delta}} = \gamma(\vc{u}+ \beta \vc{g})$ where 
\begin{equation}
    \vc{g} = -y \frac{|\vc{w}_m|^{q-1} \odot \text{sgn}(\vc{w}_m)}{\norm{\vc{w}_m}_q^{q-1}}
\end{equation}
and 
\begin{equation}
    \gamma = \frac{\epsilon}{\pnorm{\vc{u} + \beta \vc{g}}{p}} >0
\end{equation}
to ensure $\pnorm{\hat{\bm{\delta}}}{p} = \epsilon$.

Since $\zeta_m < \epsilon$, the choice of $\beta$ for the $m^{\text{th}}$ iteration will be:
\begin{equation}
    \beta = \frac{\epsilon}{\epsilon-\zeta_m}\left|\frac{y\tp{\vc{w}_m}\vc{u}}{\pnorm{\vc{w}_m}{q}} + \zeta_m \right| + \rho
\end{equation}
where $\rho>0$ is a fixed arbitrarily small positive number. Note that $\pnorm{\vc{u}}{p}$ can either be $0$ or $\epsilon$ by design. If $\pnorm{\vc{u}}{p}=0$, then the same proof from case 1 follows, otherwise assume  $\pnorm{\vc{u}}{p}=\epsilon$.

We will now show that $\hat{\bm{\delta}}$ satisfies both necessary and sufficient conditions to fool $f_m$. First, we compute the dot-product:
\begin{align}
\begin{split}
     -y\tp{\vc{w}_m}\hat{\bm{\delta}} &= -y\gamma\tp{\vc{w}_m}\left(\vc{u} + \beta \vc{g}\right) = -y\gamma\tp{\vc{w}_m}\vc{u} - y\beta\gamma\tp{\vc{w}_m}\vc{g} = -y\gamma\tp{\vc{w}_m}\vc{u} +\beta\gamma\pnorm{\vc{w}_m}{q} 
\end{split}
\end{align}
Therefore, in order to satisfy the first condition in \eqref{eq:fool-cond}, we need $-y\tp{\vc{w}_m}\hat{\bm{\delta}}$ to be strictly positive, which is satisfied by any $\beta$ satisfying:
\begin{equation}
    \beta > \frac{y\tp{\vc{w}_m}\vc{u}}{\pnorm{\vc{w}_m}{q}}
\end{equation}
However, using the fact that $\epsilon>\zeta_m>0$, our choice of $\beta$ satisfies this constraint since:
\begin{equation}
    \beta = \frac{\epsilon}{\epsilon-\zeta_m}\left|\frac{y\tp{\vc{w}_m}\vc{u}}{\pnorm{\vc{w}_m}{q}} + \zeta_m \right| +\rho> \frac{\epsilon}{\epsilon-\zeta_m}\left(\frac{y\tp{\vc{w}_m}\vc{u}}{\pnorm{\vc{w}_m}{q}} + \zeta_m \right) > \frac{y\tp{\vc{w}_m}\vc{u}}{\pnorm{\vc{w}_m}{q}}
\end{equation}

Second, we compute the following quantity required for the second condition:
\begin{align}
    \begin{split}
    \frac{\left|\tp{\vc{w}_m}\hat{\bm{\delta}}\right|}{\pnorm{\vc{w}_m}{q}} = \frac{-y\tp{\vc{w}_m}\hat{\bm{\delta}}}{\pnorm{\vc{w}_m}{q}} = \frac{-y\gamma\tp{\vc{w}_m}\vc{u} +\beta\gamma\pnorm{\vc{w}_m}{q}}{\pnorm{\vc{w}_m}{q}} 
    \end{split}
\end{align}
In order to satisfy the second condition, we need:
\begin{equation}
    \frac{-y\gamma\tp{\vc{w}_m}\vc{u} +\beta\gamma\pnorm{\vc{w}_m}{q}}{\pnorm{\vc{w}_m}{q}} > \zeta_m
\end{equation}
which is equivalent to:
\begin{equation} \label{eq:beta-ineq}
    \beta > \frac{y\tp{\vc{w}_m}\vc{u}}{\pnorm{\vc{w}_m}{q}} + \frac{\zeta_m}{\gamma}
\end{equation}
Recall that:
\begin{equation} \label{eq:gamma-ineq}
    \gamma = \frac{\epsilon}{\pnorm{\vc{u}+\beta \vc{g}}{p}} \geq \frac{\epsilon}{\pnorm{\vc{u}}{p} + |\beta|\pnorm{\vc{g}}{p}} = \frac{\epsilon}{\epsilon + \beta} 
\end{equation}
where the first inequality stems from the triangle inequality of the $\ell_p$ norm, and the second equality uses the fact that $\pnorm{\vc{u}}{p}=\epsilon$, $\pnorm{\vc{g}}{p}=1$, and $\beta>0$ by design.
Using the inequality in \eqref{eq:gamma-ineq}, we can modify the inequality condition in \eqref{eq:beta-ineq} to:
\begin{equation}
    \beta >  \frac{y\tp{\vc{w}_m}\vc{u}}{\pnorm{\vc{w}_m}{q}} + \frac{\zeta_m(\epsilon+\beta)}{\epsilon}  \geq \frac{y\tp{\vc{w}_m}\vc{u}}{\pnorm{\vc{w}_m}{q}} + \frac{\zeta_m}{\gamma} 
\end{equation}
which means it is enough to require that:
\begin{equation}
    \beta\left(\frac{\epsilon - \zeta_m}{\epsilon}\right) > \frac{y\tp{\vc{w}_m}\vc{u}}{\pnorm{\vc{w}_m}{q}} + \zeta_m
\end{equation}
for the second condition to hold. It is easy to see our choice of $\beta$ satisfies this strict inequality (due to $\rho >0$). In fact, we can easily show that $\frac{y\tp{\vc{w}_m}\vc{u}}{\pnorm{\vc{w}_m}{q}} + \zeta_m$ is guaranteed to be positive in this scenario due to our assumption that $\vc{x}+\vc{u}$ is correctly classified by $f_m$, which removes the looseness of the absolute value in our choice of $\beta$. Thus we have shown that, given the assumptions, we can find a $m\in[M]$ such that $\vc{x}+\bm{\delta}^{(m)}$ is misclassified by $f_m$, which implies $v^{(m)}<v^{(m-1)}=1$, which concludes the proof that the ARC algorithm is consistent.
\end{proof}
\clearpage
\section{Experimental Setup Details}\label{app:setup}
\subsection{Training Hyper-parameters}
In this section, we provide the training details alongside the choice of hyper-parameters for all our experiments. For SVHN, CIFAR-10, and CIFAR-100 datasets, we establish strong adversarially trained baselines (via PGD) following the approach of \cite{rice2020overfitting} which utilizes early stopping to avoid robust over-fitting. As for ImageNet, we use the popular free training (FT) \cite{freetraining} method for faster and more efficient adversarial training. We employ standard data augmentation techniques (random cropping and random horizontal flips) for all models. Following standard practice, we also apply input normalization as part of the model, so that the adversary operates on physical images $\vc{x} \in [0,1]^D$. For all our experiments, the same $\ell_p$ norm is used during both training and evaluation.

For BAT training, as mentioned in Section~\ref{sec:results}, we re-use the AT models as the first model ($f_1$) in the REC, and then train a second model ($f_2$) using the adversarial samples of the \emph{first} model only. The same network architecture, training algorithm, and hyper-parameters will be used for both models. A single workstation with two NVIDIA Tesla P100 GPUs is used for running all the training experiments. Below are the hyper-parameter details for each dataset:

\textbf{SVHN}: Models are trained for $200$ epochs, using a PGD adversary with $K=7$ iterations with: $\epsilon = 8/255$ and  $\eta=2/255$ for $\ell_\infty$ AT, and  $\epsilon = 128/255$ and  $\eta=32/255$ for $\ell_2$ AT. We use stochastic gradient descent (SGD) with momentum ($0.9$), $128$ mini-batch size, and a step-wise learning rate decay set initially at $0.1$ and divided by
$10$ at epochs $100$ and $150$. We employ weight decay of $2\times 10^{-4}$.

\textbf{CIFAR-10}: Models are trained for $200$ epochs, using a PGD adversary with $K=7$ iterations with: $\epsilon = 8/255$ and  $\eta=2/255$ for $\ell_\infty$ AT, and  $\epsilon = 128/255$ and  $\eta=32/255$ for $\ell_2$ AT. We use SGD with momentum ($0.9$), $128$ mini-batch size, and a step-wise learning rate decay set initially at $0.1$ and divided by
$10$ at epochs $100$ and $150$. We employ weight decay of: $2\times 10^{-4}$ for ResNet-20 and MobileNetV1, and $5\times 10^{-4}$ for VGG-16, ResNet-18, and WideResNet-28-4.

\textbf{CIFAR-100}: Models are trained for $200$ epochs, using a PGD adversary with $K=7$ iterations with: $\epsilon = 8/255$ and  $\eta=2/255$ for $\ell_\infty$ AT, and  $\epsilon = 128/255$ and  $\eta=32/255$ for $\ell_2$ AT. We use SGD with momentum ($0.9$), $128$ mini-batch size, and a step-wise learning rate decay set initially at $0.1$ and divided by
$10$ at epochs $100$ and $150$. We employ weight decay of $2\times 10^{-4}$.

\textbf{ImageNet}: Models are trained for a total of $90$ epochs with mini-batch replay $m=4$, using: $\epsilon = 4/255$ for $\ell_\infty$ FT, and  $\epsilon = 76/255$ for $\ell_2$ FT. We use stochastic gradient descent (SGD) with momentum ($0.9$), $256$ mini-batch size and  a step-wise learning rate decay set initially at $0.1$ and decayed by $10$ every $30$ epochs. We employ weight decay of $1\times 10^{-4}$.

\subsection{Evaluation Details}
We evaluate the robustness of standard AT models using PGD \cite{madry2018towards}. As for randomized ensembles, we establish baselines for comparison using the adaptive version of PGD (APGD) \cite{pinot2020randomization} which uses the expected loss function (see Section~\ref{sec:preliminaries}). We use our proposed ARC algorithm as well, to demonstrate the vulnerability of RECs. In Appendix~\ref{app:attacks}, we investigate different methods for constructing adaptive and non-adaptive attacks using both PGD and C\&W \cite{carlini2017towards}. The PGD/APGD attack details for each dataset are listed below. The same configurations are used for ARC as well, except for $\ell_\infty$ evaluations, where we find that setting $\eta=\epsilon$ yields best results.

\textbf{SVHN, CIFAR-10, CIFAR-100}: We use $K=20$ iterations, with: $\epsilon=8/255$ with $\eta=2/255$ for $\ell_\infty$ PGD and $\epsilon=128/255$ with $\eta=32/255$ $\ell_2$ PGD. 

\textbf{ImageNet}: We use $K=50$ iterations, with: $\epsilon=4/255$ with $\eta=1/255$ for $\ell_\infty$ PGD and $\epsilon=128/255$ with $\eta=32/255$ for $\ell_2$ PGD.


\clearpage
\section{Additional Experiments and Comparisons}\label{app:extra}

\subsection{Individual Model Performance}
\begin{table}[tbhp]
\centering
\caption{Clean and robust accuracies of the individual members of
BAT trained RECs from Table~\ref{tab:datasets}. Model robustness is measured via
both PGD and ARC.} \label{tab:individual-models}
\vskip 0.15in
\begin{sc}
\resizebox{0.7\columnwidth}{!}{%
\begin{tabular}{l c c  c r r }

\toprule
\multirow{2}{*}{Dataset}  & \multirow{2}{*}{Norm}  & \multirow{2}{*}{Model }  & \multirow{2}{*}{Clean Accuracy [\%] }& \multicolumn{2}{c}{Robust Accuracy [\%]} \\
  &   & & & PGD & ARC\\
\midrule
\multirow{4}{*}{SVHN}   &  \multirow{2}{*}{$\ell_2$} & $f_1$   & $93.07$ & $68.35$  & $65.50$ \\
   &   & $f_2$   & $91.23$ & $0.00$  & $0.00$ \\
\cmidrule(lr{1em}){2-6}
  &  \multirow{2}{*}{$\ell_\infty$} & $f_1$   & $90.53$ & $53.55$  & $49.01$ \\
   &   & $f_2$   & $88.31$ & $0.00$  & $0.00$ \\
\midrule
\multirow{4}{*}{CIFAR-10}   &  \multirow{2}{*}{$\ell_2$} & $f_1$   & $83.36$ & $62.43$  & $61.13$ \\
   &   & $f_2$   & $91.26$ & $0.13$  & $0.00$ \\
\cmidrule(lr{1em}){2-6}
  &  \multirow{2}{*}{$\ell_\infty$} & $f_1$   & $75.96$ & $45.66$  & $42.36$ \\
   &   & $f_2$   & $77.64$ & $0.00$  & $0.00$ \\
\midrule
\multirow{4}{*}{CIFAR-100}   &  \multirow{2}{*}{$\ell_2$} & $f_1$   & $53.43$ & $34.60$  & $31.88$ \\
   &   & $f_2$   & $63.74$ & $0.00$  & $0.00$ \\
\cmidrule(lr{1em}){2-6}
  &  \multirow{2}{*}{$\ell_\infty$} & $f_1$   & $44.37$ & $22.29$  & $18.36$ \\
   &   & $f_2$   & $34.1$ & $0.00$  & $0.00$ \\
\midrule
\multirow{4}{*}{ImageNet}   &  \multirow{2}{*}{$\ell_2$} & $f_1$   & $62.91$ & $47.60$  & $45.96$ \\
   &   & $f_2$   & $64.60$ & $0.00$  & $0.00$ \\
\cmidrule(lr{1em}){2-6}
  &  \multirow{2}{*}{$\ell_\infty$} & $f_1$   & $52.01$ & $24.33$  & $21.03$ \\
   &   & $f_2$   & $55.70$ & $0.00$  & $0.00$ \\
\bottomrule

\end{tabular}
}
\end{sc}
\end{table}
\begin{figure}[!b]
  \centering
  \includegraphics[width=0.7\columnwidth]{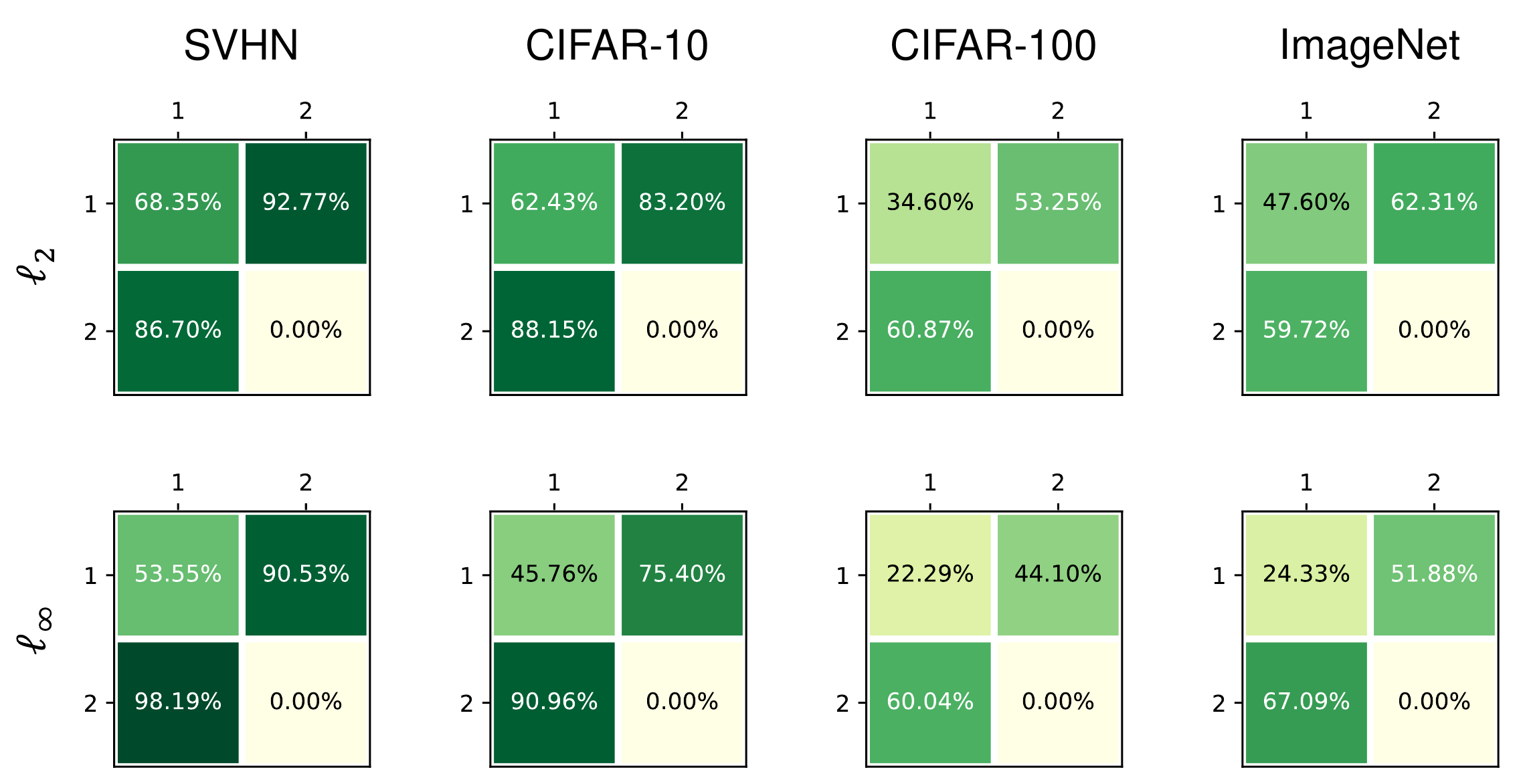}
  \caption{The cross-robustness matrices of all BAT REC from Table~\ref{tab:datasets}. As expected, BAT training induces an asymmetric ensemble were the first model is robust and the second model is completely compromised. The \emph{false} sense of robustness from BAT RECs stems from the high robustness of the second model to the adversarial samples of the first model.}
  \label{fig:cross-rob-bat}
\end{figure}

In this section, we provide more details on the performance of the individual models from the BAT RECs constructed in Table~\ref{tab:datasets} across different datasets. Table~\ref{tab:individual-models} reports the clean and robust accuracies of all the individual models. The robust accuracy is measured via both standard PGD and our ARC adversary. We consistently find that the performance of ARC and PGD in the standard single model setting to be similar, with ARC being slightly better. One thing to note is that, due to the nature of BAT, the second model $f_2$ is always \emph{not} robust to its own adversarial perturbation. However, as shown in Fig.~\ref{fig:cross-rob-bat}, the second model is very successful at defending against the adversarial samples of the first model, which is expected.


\subsection{More Parameter Sweeps}\label{app:sweeps}
In this section, we expand on the results in Section~\ref{ssec:tests}. Specifically, Fig.~\ref{fig:more-sweeps} provides a more complete version of Fig.~\ref{fig:comp} with added CIFAR-100 results, across two norms $\ell_\infty$ and $\ell_2$.  Across all datasets, norms, and parameter values, Fig.~\ref{fig:more-sweeps} consistently shows that the randomized ensembles obtained via BAT are more vulnerable than originally claimed. The ARC algorithm is much more effective at generating adversarial examples for RECs, in contrast to APGD.
\begin{figure}[hpb]
  \centering
    \includegraphics[width=0.9\columnwidth]{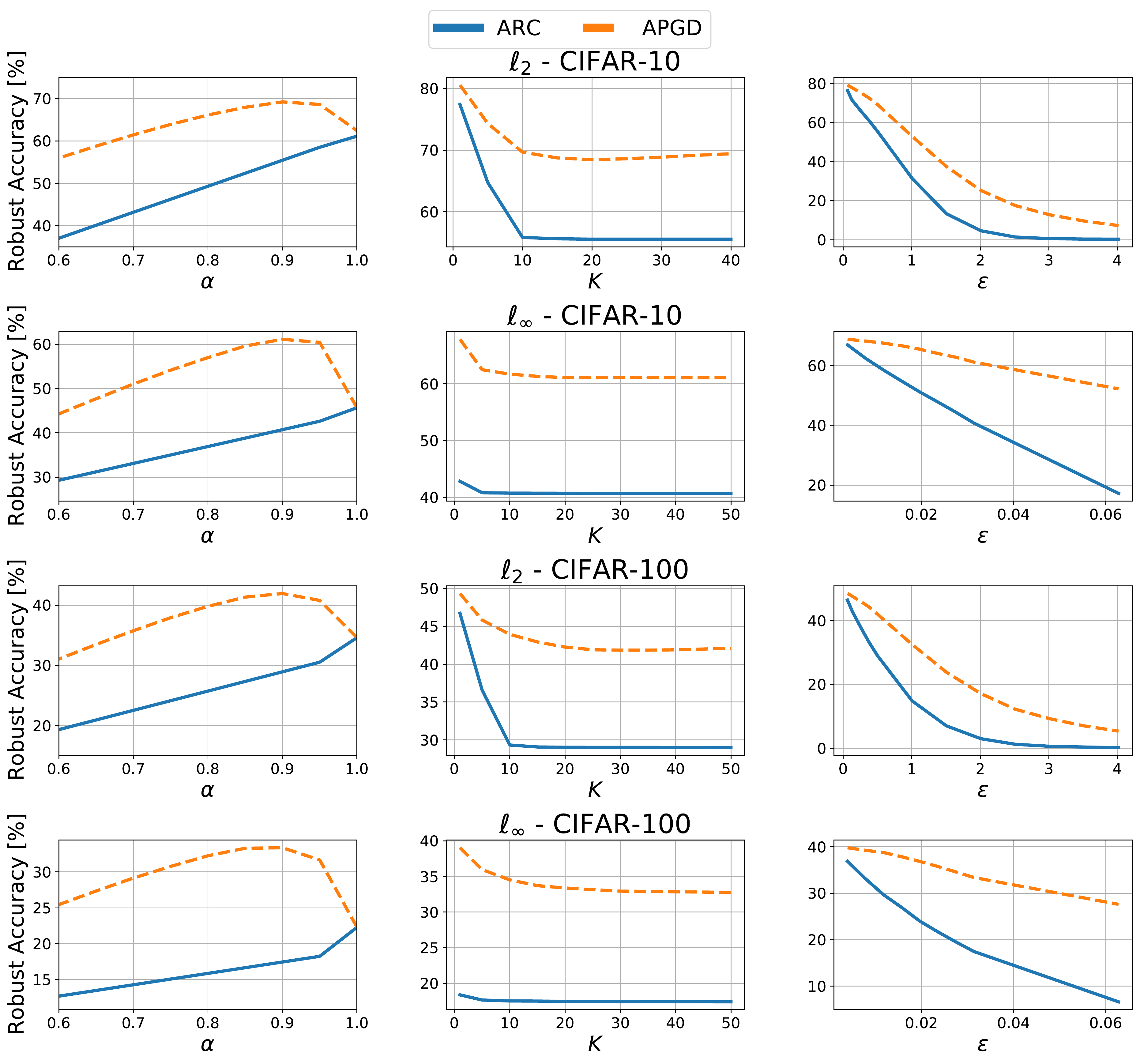}
  \caption{More comparisons between ARC and APGD using RECs of ResNet-20s obtained via BAT ($M=2$).}
  \label{fig:more-sweeps}
\end{figure}

\subsection{More on Adaptive Attacks}\label{app:attacks}
In this section, we explore different ways of deploying PGD-based adversaries for attacking RECs and compare them with ARC. We also employ the C\&W \cite{carlini2017towards} attack for $\ell_2$ norm-bounded perturbations, and adapt it to our setting.  
\subsubsection{Adaptive PGD}
In Section~\ref{sec:preliminaries}, we introduced the standard PGD attack, which performs the following update rule $\forall k \in [K]$:

\begin{equation} \label{eq:pgd-r}
    \bm{\delta}^{(k)} = \Pi_\epsilon^p \left( \bm{\delta}^{(k-1)} + \eta \mu_p\left(\nabla_{\vc{x}} l(f(\vc{x}+\bm{\delta}^{(k-1)}) ,y)\right)\right)
\end{equation}
where $l$ is the loss function, $f$ is the differentiable deterministic classifier, and $(\vc{x},y)$ is the labeled data-point. It is clear that some changes have to made in order to use PGD for a randomized ensemble $(\calF,\bm{\alpha})$. We propose some intuitive changes below.

\textbf{PGD}: Perhaps the simplest modification one can make is to randomize the attack as well! If the defender picks a model $f_i$ at random from the ensemble with probability $\alpha_i$, for some $i\in [M]$, then the attacker can also \textit{independently} sample a model $f_j$ at random with probability $\alpha_j$ and use it to construct the adversarial perturbation for $(\calF,\bm{\alpha})$ via PGD. Again, the discrete nature of our setup allows us to compute exactly the expected robust accuracy, when both the attacker and defender are randomizing their strategies:
\begin{equation}
    \bar{L} = \sum_{j=1}^M \alpha_j L(\vc{x}+\bm{\delta}_j,\bm{\alpha})
\end{equation}
where $\bm{\delta}_j$ is the adversarial perturbation obtained via PGD when attacking model $f_j$, and $L$ is the conditional expected robust accuracy from \eqref{eq:acc}.

\textbf{PGD-1}: Since the second model in BAT RECs is always not robust, one logical attack to consider is attacking the robust model only $f_1$, and ignore the second non-robust model. That is, we use standard PGD to attack $f_1$, and use the generated $\bm{\delta}$ to attack the REC.

\textbf{APGD}: The attacker can also adapt PGD by using the expected value of the loss function, and modifying \eqref{eq:pgd-r} as follows:

\begin{align}
\begin{split}
  \bm{\delta}^{(k)} &= \Pi_\epsilon^p \left( \bm{\delta}^{(k-1)} + \eta \mu_p\left(\nabla_{\vc{x}} \mean{l(f(\vc{x}+\bm{\delta}^{(k-1)}) ,y)}\right)\right)  \\
  &= \Pi_\epsilon^p \left( \bm{\delta}^{(k-1)} + \eta \mu_p\left( \sum_{i=1}^{M}\left[\alpha_i \nabla_{\vc{x}}l(f_i(\vc{x}+\bm{\delta}^{(k-1)}) ,y)\right]\right)\right)
\end{split}
\end{align}

Recall this is the adaptive PGD proposed in \cite{pinot2020randomization} and what we adopted in this paper (see Sections~\ref{sec:preliminaries} and \ref{sec:pgd}) for robustness evaluation of randomized ensembles.

\textbf{APGD-L}: Another method for computed the EOT for randomized ensembles is to compute the expected classifier outputs (logits), instead of the expected loss, which will yield the following change to \eqref{eq:pgd-r}:
\begin{align}
\begin{split}
  \bm{\delta}^{(k)} &= \Pi_\epsilon^p \left( \bm{\delta}^{(k-1)} + \eta \mu_p\left(\nabla_{\vc{x}} l(\mean{f(\vc{x}+\bm{\delta}^{(k-1)})} ,y)\right)\right)  \\
  &= \Pi_\epsilon^p \left( \bm{\delta}^{(k-1)} + \eta \mu_p\left(\nabla_{\vc{x}} l\left(\sum_{i=1}^M \alpha_i f_i(\vc{x}+\bm{\delta}^{(k-1)}) ,y\right)\right)\right)
\end{split}
\end{align}
The authors of \cite{pinot2020randomization} were also aware of this modification. However, they argued that averaging the logits may cause discrepancies, since logits from two different models can have completely different ranges for values. The compelling argument for using the loss function, is that it provides an elegant and natural way of normalizing the outputs. Note that, our analysis in Section~\ref{sec:pgd} is also applicable for this version of adaptive PGD, and therefore we expect it to inherit the shortcomings of APGD and provide a \emph{false} sense of robustness for RECs.

\subsubsection{C\&W Attack}
The C\&W attack \cite{carlini2017towards} is another popular and powerful attack originally designed for compromising defensive distillation \cite{papernot2016distillation} in neural networks. For a given classifier $f$ and data-point $(\vc{x},y)$, the C\&W attack finds an adversarial perturbation by solving the following constraint optimization problem:
\begin{equation} \label{eq:cw}
    \bm{\delta}^* = \argmin_{\bm{\delta}+\vc{x} \in [0,1]^D} \pnorm{\bm{\delta}}{2} + c h(\vc{x}+\bm{\delta})
\end{equation}
where $h(\vc{x}+\bm{\delta})$ is some cost function such that $h(\vc{x}+\bm{\delta}) <0$ if and only if $\argmax [f(\vc{x}+\bm{\delta})]_i \neq y$, and $c>0$ is a constant that is often optimized via binary search. Note that the C\&W attack \cite{carlini2017towards} tries to find the minimum $\ell_2$ norm perturbation that fools $f$, such that the perturbed image satisfies the box constraints. The resulting perturbation therefore is not guaranteed to be norm-bounded by some $\epsilon$. However, a common workaround is to simply project the final perturbation onto the $\ell_2$ ball of radius $\epsilon$.

Solving \eqref{eq:cw} is very challenging from an optimization point of view. To facilitate the optimization, the authors propose a change of variable:
\begin{equation}
    \bm{\delta} = \frac{1}{2}(\text{tanh}(\bm{\theta})+1) - \vc{x}
\end{equation}
where $\text{tanh}$ is an element-wise hyperbolic tangent, which guarantees that $\bm{\delta} + \vc{x}$ satisfies the box constraints $\forall \bm{\theta} \in \reals^D$. Therefore, the C\&W attack solves for $\bm{\theta}^*$ via:
\begin{equation}\label{eq:cw-theta}
    \bm{\theta}^* = \argmin_{\bm{\theta} \in \reals^D} \pnorm{\frac{1}{2}(\text{tanh}(\bm{\theta})+1) - \vc{x}}{2}^2 + c h(\frac{1}{2}(\text{tanh}(\bm{\theta})+1))
\end{equation}
where $h$ is:
\begin{equation}
    h(\vc{u}) = \max\left\{[f(\vc{u})]_y - \max_{i\in [C]\setminus \{y\}} [f(\vc{u})]_i, -\kappa\right\}
\end{equation}
and $\kappa$ controls the confidence (usually set to $0$). Finally, the optimization in \eqref{eq:cw-theta} can be solved via gradient based algorithms, such as SGD or ADAM \cite{kingma2014adam}, while $c>0$ is chosen using binary search. 

Analogous to adapting PGD, we will experiment with three version of the C\&W attack for RECs: 

\textbf{C\&W}: The attacker randomly samples a classifier and attacks it via the classical C\&W adversary.

\textbf{C\&W-1}: The attacker only considers the first robust model (for BAT RECs) via the classical C\&W adversary, akin to PGD-1.

\textbf{AC\&W}: The attacker uses the expected value of $h$ when solving \eqref{eq:cw-theta}, akin to APGD.

\textbf{AC\&W-L}: The attacker uses the expected output of the REC (exepected logits) when solving \eqref{eq:cw-theta}, akin to APGD-L.

To ensure proper implementation, we adopt the open-source Foolbox \cite{rauber2017foolbox} implementation of the C\&W attack. We adopt the common attack hyper-parameters: $9$ steps for binary search with an initial constant $c=0.001$, optimizer learning rate $0.01$, confidence $\kappa = 0$, and run the attack for a total of $K=50$ iterations. 

\begin{figure}[t]
  \centering
    \includegraphics[width=\columnwidth]{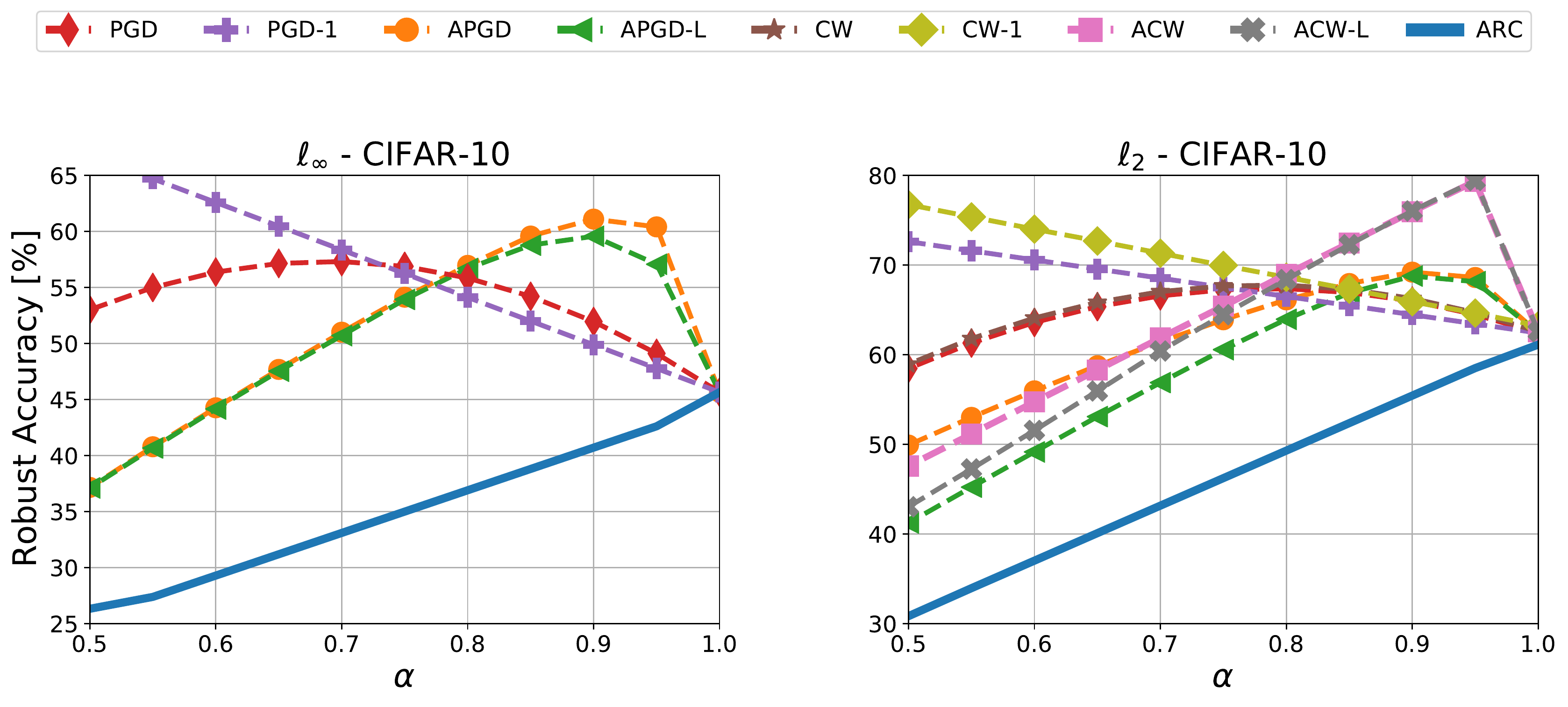}
  \caption{Robust accuracy vs. sampling probability of: $\ell_\infty$ (left) and $\ell_2$ (right) BAT REC of ResNet-20s on CIFAR-10. Robust accuracy is evaluated using various adaptive attacks detailed in Appendix~\ref{app:attacks} and the proposed ARC attack algorithm.}
  \label{fig:more-attacks}
\end{figure}

\subsubsection{ARC vs. Adaptive Attacks}
We now apply all proposed variants of PGD and C\&W for attacking the RECs of ResNet-20s on CIFAR-10 (from Section~\ref{ssec:tests}). Note that the C\&W variants will only be evaluated against the $\ell_2$ trained REC. Figure.~\ref{fig:more-attacks} plots the robust accuracy of all methods while sweeping the defender's sampling probability $\bm{\alpha}=(\alpha,1-\alpha)$. 
We make the following observations:
\begin{itemize}
    \item The ARC algorithm outperforms all variants of PGD and C\&W and across all values of $\alpha$, and by massive margins ($\sim 20\%$) for $\alpha \in [0.5, 0.9]$.
    \item ARC demonstrates that there is no benefit in using RECs, since the highest robustness achieved corresponds to $\alpha=1$. Note that $\alpha=1$ reduces the REC to the deterministic adversarially trained classifier $f_1$.
    \item Adapting PGD via the expected loss, instead of computing the expected logits yields little to no difference, and the same can be said for C\&W as well.
    \item Interestingly, attacking the first model only is a much more effective attack than any exsitng adaptive attack in the high sampling probabilty regime ($\alpha\in[0.8,1]$). Despite this fact, ARC remains a much stronger attack for all values of $\alpha$. 
    \item The robustness evaluation of APGD and PGD follow completely different trends, with both providing \emph{false} sense of robustness. The same can be said for AC\&W and CW as well.
\end{itemize}

\subsection{More on Ensemble Training Methods} \label{app:ensemble}
In this section, we expand on the results from Section~\ref{ssec:ensemble}, where we constructed RECs from pre-trained DVERGE models \cite{yang2020dverge}. In addition to DVERGE, we experiment with ADP \cite{pang2019improvingADP} and TRS \cite{yang2021trs} using an ensemble of three ResNet-20s on CIFAR-10. We used the pre-trained ADP models that were publicly available thanks to \cite{yang2020dverge}, and trained TRS models using their publicly released code on GitHub \cite{yang2021trs}.

We first plot the cross-robustness matrices for each ensemble in Fig.~\ref{fig:static-ens-mat}. None of the classifiers in the ensembles are \emph{actually} robust, which is expected. These training methods are designed to improve the robustness to transfer attacks, which is why we notice high cross-robustness in all three matrices. 

We construct RECs for each method via equiprobable sampling, and evaluate the robustness against both APGD and ARC in Fig.~\ref{fig:static-ens-sweep} for different norm radius $\epsilon$ values. A common observation to all three methods is that ARC consistently outperforms APGD, and with large margins for TRS and DVERGE. These experiments demonstrate that RECs are \emph{vulnerable}, regardless of the training procedure thus motivating the need for improved adversarial training algorithms for randomized ensembles. 

\begin{figure}[h]
  \centering
    \subfloat[]{\includegraphics[height=4.5cm]{figures/cross_rob_cifar10_linf_resnet20_dverge_m3.pdf}\label{fig:dverge-mat}}%
    \qquad%
    \subfloat[]{\includegraphics[height=4.5cm]{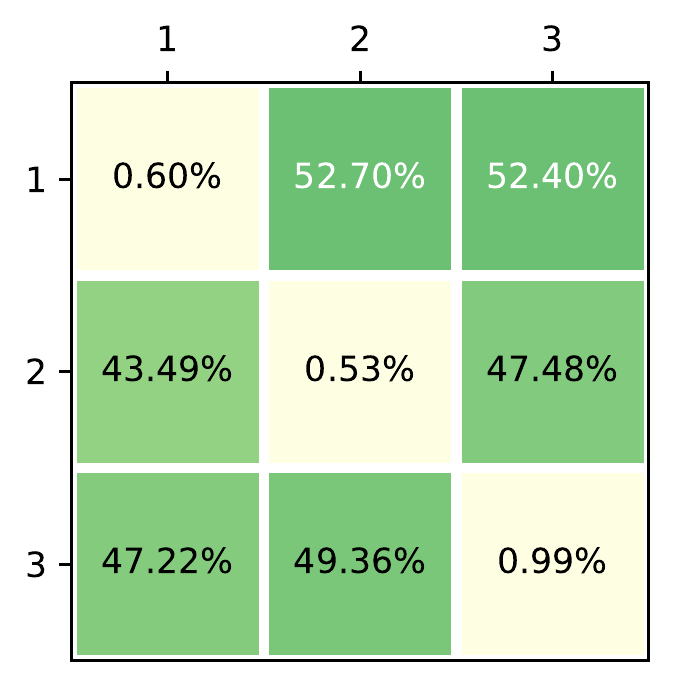}\label{fig:adp-mat}}%
    \qquad%
    \subfloat[]{\includegraphics[height=4.5cm]{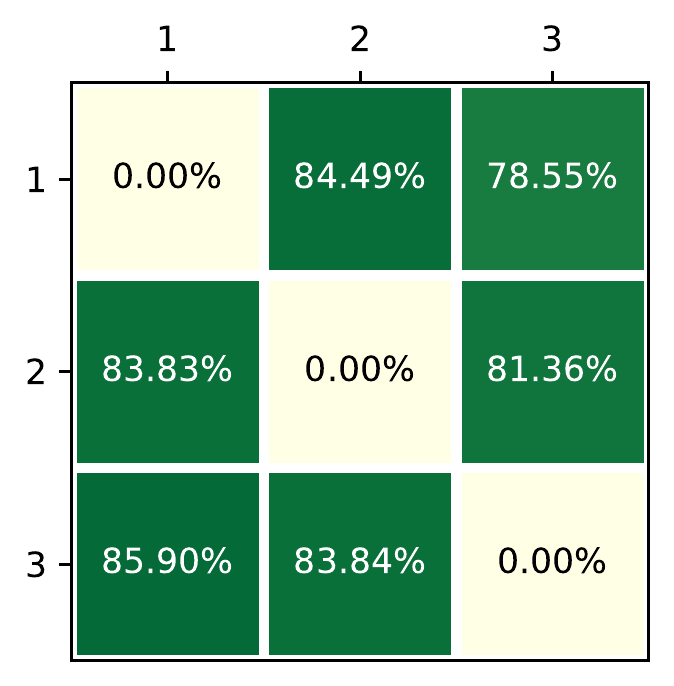}\label{fig:trs-mat}}
  \caption{Cross-robustness matrices for: (a) DVERGE, (b) ADP, and (c) TRS trained ensembles of three ResNet-20s on CIFAR-10.}
  \label{fig:static-ens-mat}
\end{figure}

\begin{figure}[bht]
  \centering
    \subfloat[]{\includegraphics[height=4cm]{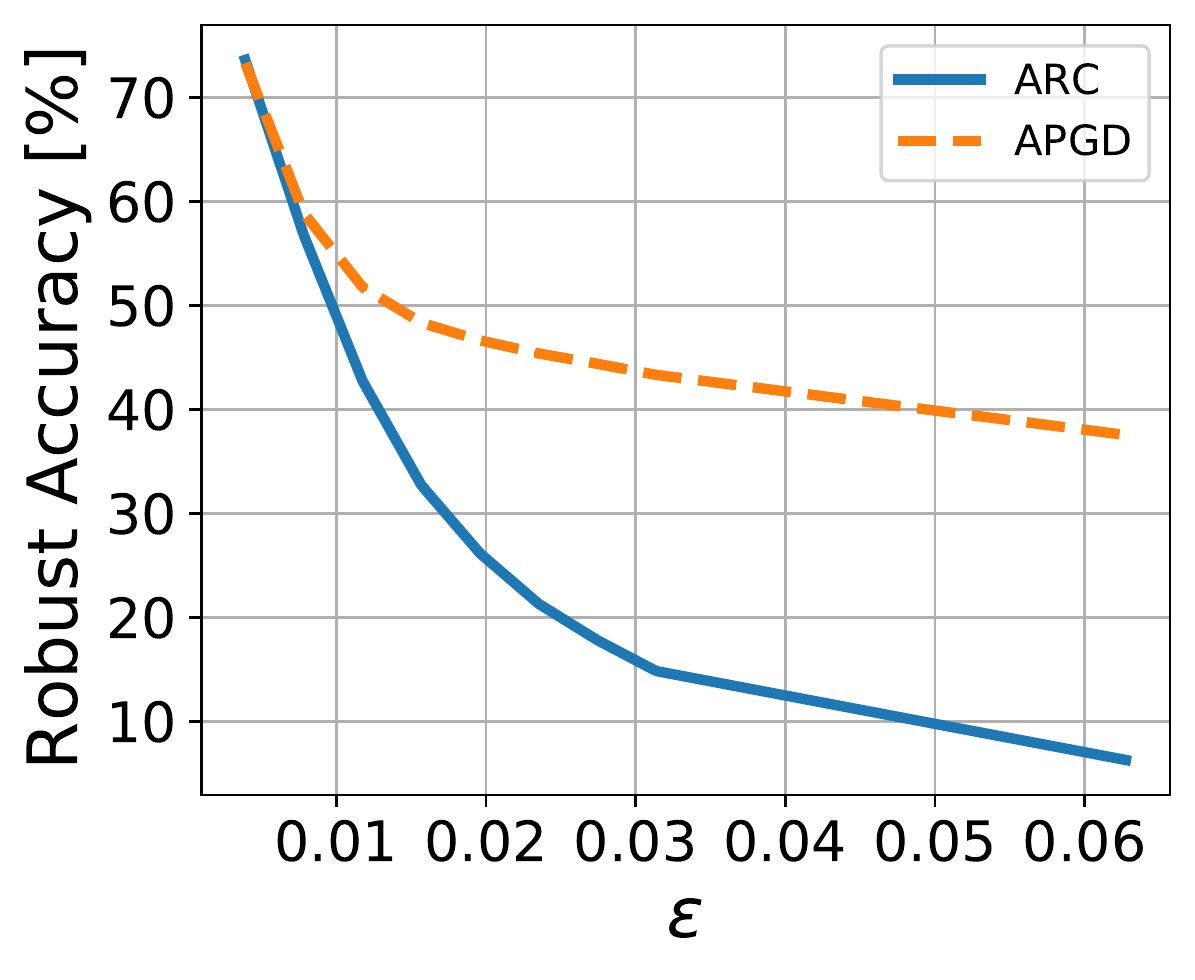}\label{fig:dverg-sweep}}%
    \qquad%
    \subfloat[]{\includegraphics[height=4cm]{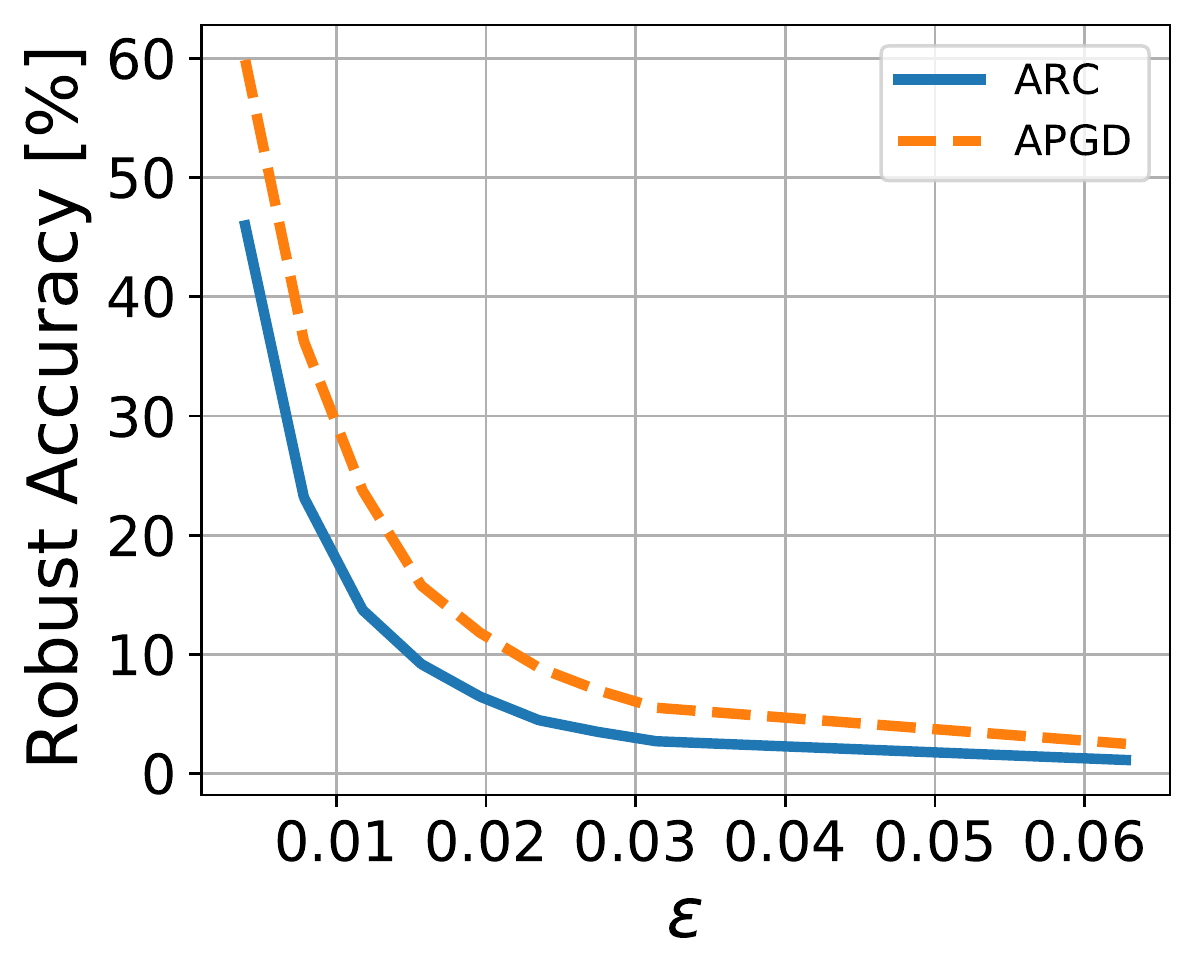}\label{fig:adp-sweep}}%
    \qquad%
    \subfloat[]{\includegraphics[height=4cm]{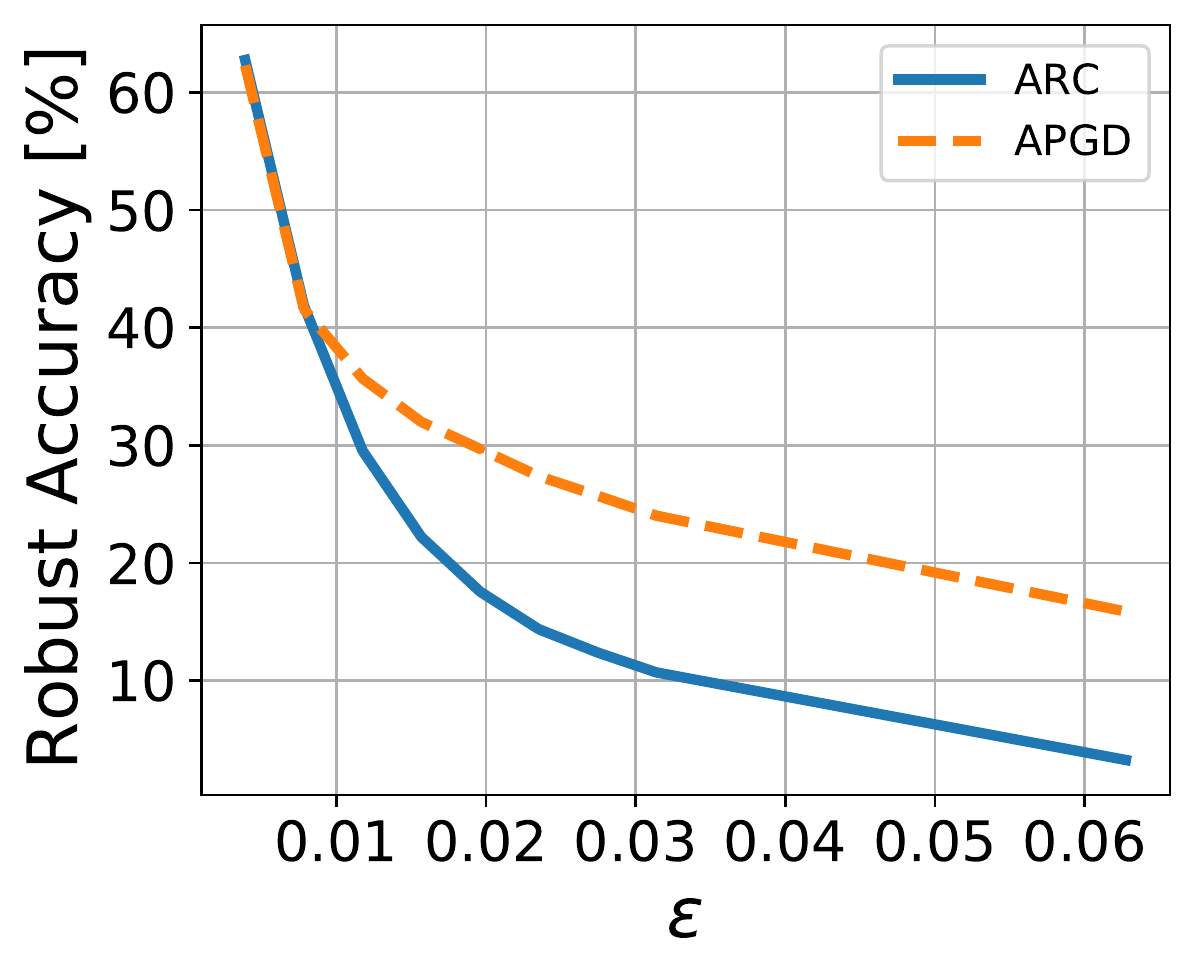}\label{fig:trs-sweep}}
  \caption{Robust accuracy vs. norm radius $\epsilon$ using both ARC and APGD for: (a) DVERGE, (b) ADP, and (c) TRS trained ensembles of three ResNet-20s on CIFAR-10.}
  \label{fig:static-ens-sweep}
\end{figure}

\subsection{Randomized Ensembles from Independent Adversarially Trained Classifiers} \label{app:iat}

In this section, we investigate the performance of RECs constructed from independent adversarially trained (IAT) deep nets. Specifically, we train $M$ ResNet-20s using $\ell_2$ AT on CIFAR-10. We use different random initializations for each network by changing the random seed. Doing so will result in a symmetric ensemble as seen via the cross-robustness matrix in Fig.~\ref{fig:indep-cross}. In this setting, all the models are robust with $\sim 62\%$ robust accuracy, as opposed to the asymmetric BAT setting where the first model is robust and the second model is completely compromised. Using equiprobable sampling, we construct RECs from the $M$ models and compare the performance of both ARC and APGD with varying ensemble size $M$ in Fig.~\ref{fig:indep-sweep}, e.g., $M=2$ means we only use the first 2 models and ignore the rest. Both ARC and APGD follow the same trend since the robustness of the ensemble increases with $M$. The ARC algorithm is consistently stronger than APGD, albeit the gap is relatively small compared to BAT. This is expected since all models in the ensemble are robust. We do note however that the improvement in robustness is limited and requires large ensemble size for it to be significant, e.g., $<3\%$ improvement with $M=8$. Finally, we strongly believe that future ensemble training methods need to employ IAT RECs as a baseline.

\begin{figure}[!hbt]
  \centering
    \subfloat[]{\includegraphics[height=6.5cm]{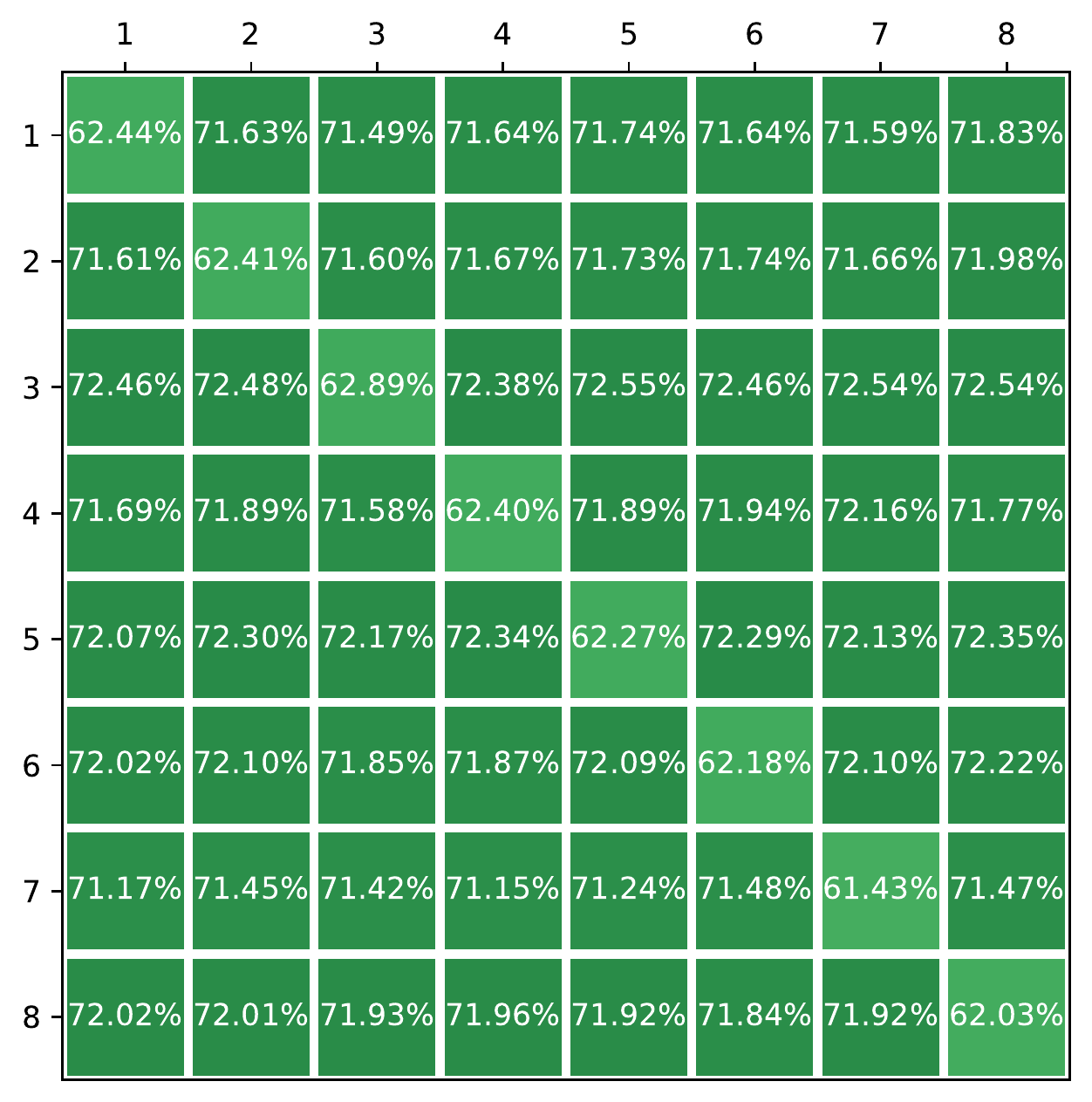}\label{fig:indep-cross}}%
    \qquad%
    \subfloat[]{\includegraphics[height=6.5cm]{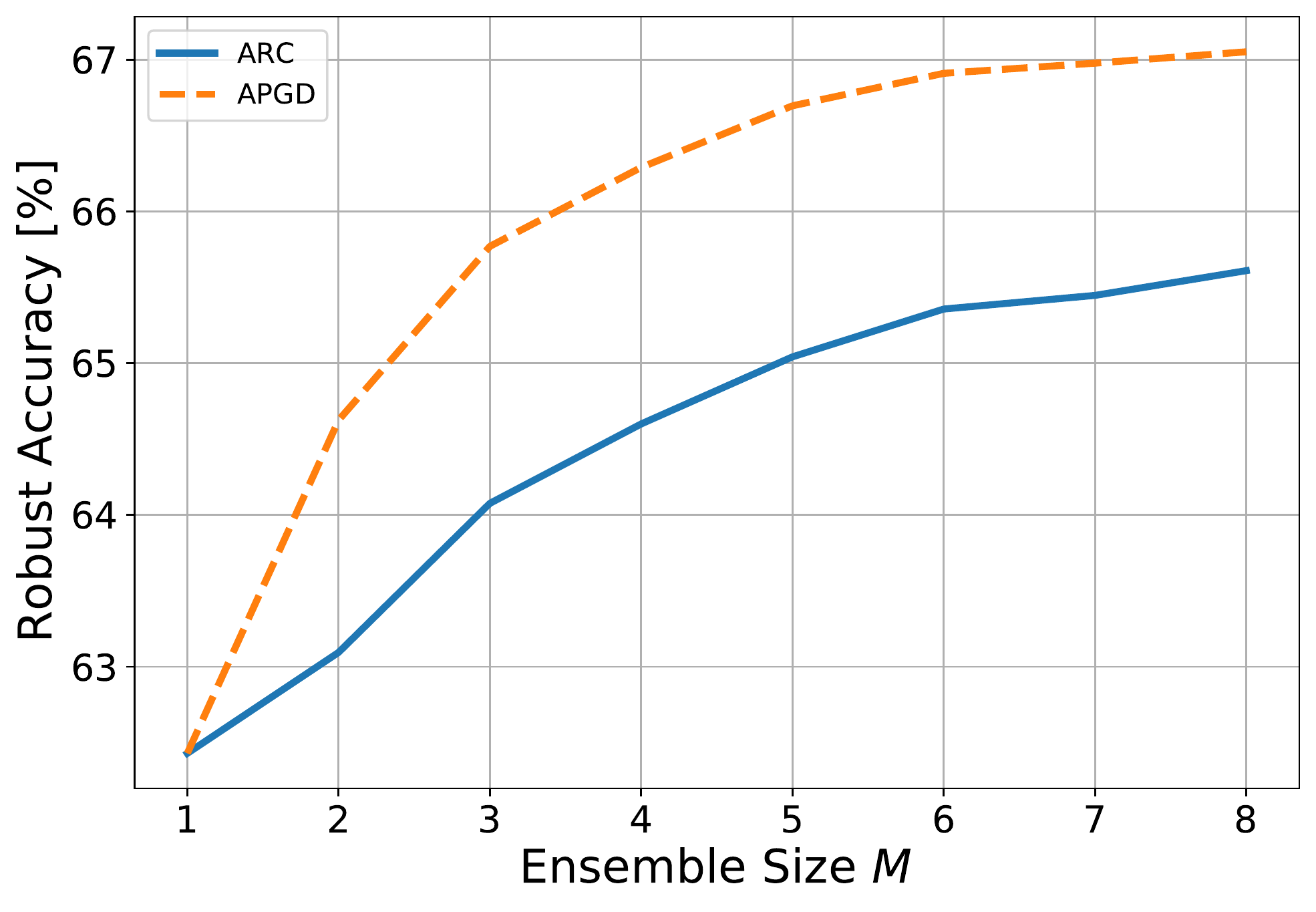}\label{fig:indep-sweep}}%
 \label{fig:indep}
  \caption{Robustness performance of an REC constructed from $\ell_2$ independent adversarially trained models ($M=8$): (a) cross-robustness matrix, and (b) robust accuracy vs. size of the ensemble $M$ using both ARC and APGD.}
\end{figure}

\subsection{Visualization of Adversarial Samples}
In this section, we provide some visual examples of $\ell_\infty$ norm-bounded adversarial perturbations obtained via APGD and ARC. Specifically, Fig.~\ref{fig:more-images} shows three examples from our ImageNet experiments, where each clean image (unperturbed) is correctly classified by a randomized ensemble of ResNet-18s (from Table~\ref{tab:datasets}). The APGD adversarial samples are also correctly classified by both networks in the ensemble, and thus are unable to fool the ensemble. In contrast, the ARC adversarial samples completely compromise the ensemble, i.e., they fool both the networks. All adversarial perturbations are $\ell_\infty$ norm-bounded with $\epsilon=4/255$.
\begin{figure}[hpb]
  \centering
    \includegraphics[width=\columnwidth]{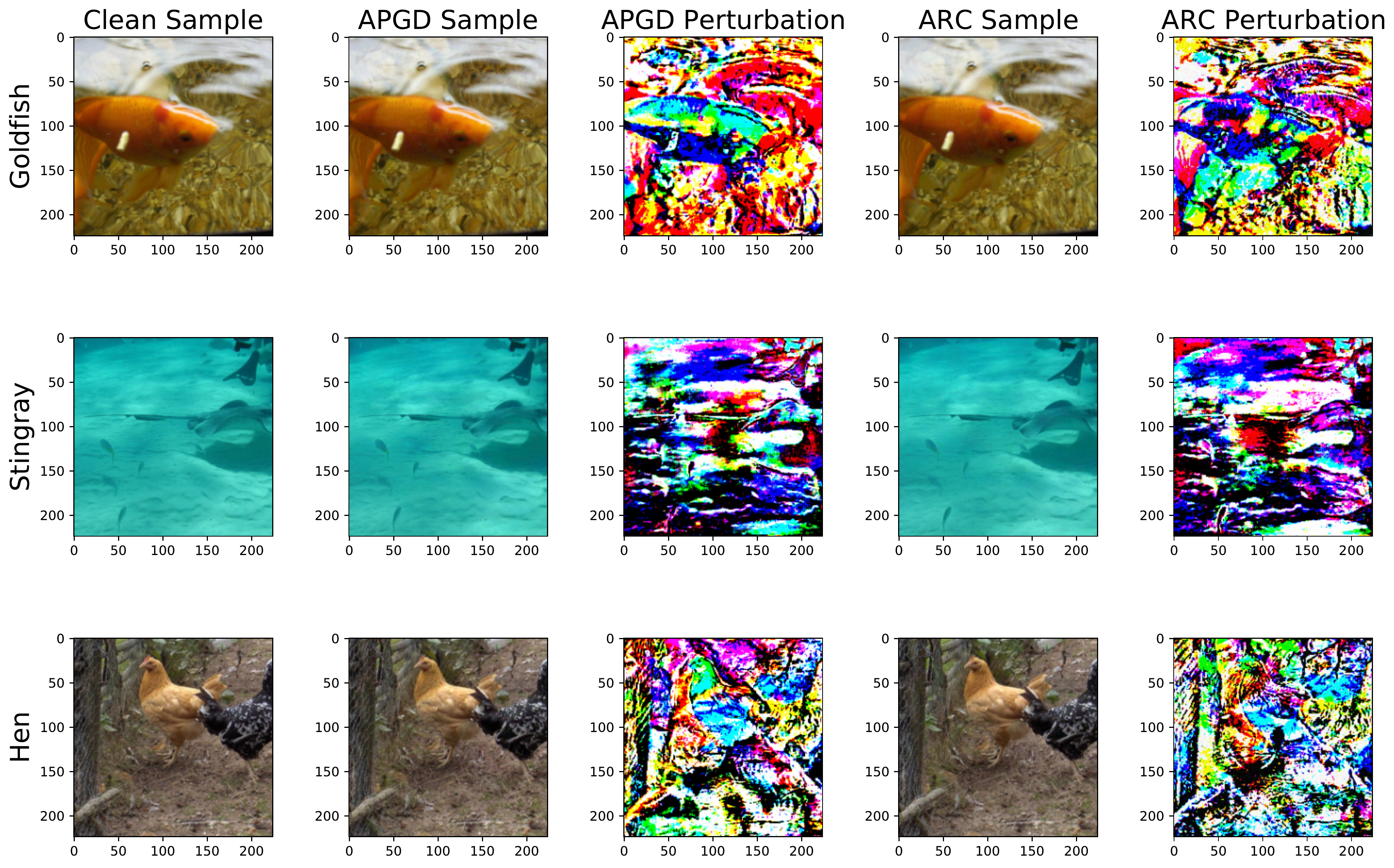}
  \caption{Visual examples of $\ell_\infty$ norm-bounded adversarial perturbations obtained via APGD and ARC evaluated against a randomized ensemble of ResNet-18s on ImageNet. All perturbations are $\ell_\infty$ norm-bounded with $\epsilon=4/255$. The APGD and clean samples are correctly classified by all members of the ensemble, whereas the ARC samples are misclassified by all members of the same ensemble.}
  \label{fig:more-images}
\end{figure}

\clearpage

\section{Scalability of the ARC Algorithm}\label{app:arclite}

\begin{table}[!b]
\centering
\caption{The robust accuracy and run time of the proposed approximate version of ARC evaluated against two RECs on CIFAR-100. The evaluation was run on a single NVIDIA 1080 Ti GPU.}
\label{tab:arclite}
\vskip 0.15in
\begin{sc}
\begin{tabular}{c c c c }

\toprule
\multirow{2}{*}{Search Size ($G$) }  & \multirow{2}{*}{Run Time [min]}& \multicolumn{2}{c}{Robust Accuracy [\%]} \\
&  & $\ell_2$& $\ell_\infty$\\

\midrule
$1$    &$2.06$ &  $29.79$ &$18.59$  \\
\midrule
$2$   & $2.56$  & $29.08$& $17.71$   \\
\midrule
$3$ & $3.05$  &$28.97$ & $17.55$ \\
\midrule
$4$  & $3.56$  & $28.92$ & $17.45$\\
\midrule
\midrule
$10$   & $6.44$  & $28.88$ & $17.32$\\
\midrule
\midrule
$50$ & $26.46$   & $28.88$ & $17.32$\\
\midrule
\midrule
$99$   & $49.91$ & $28.88$ & $17.32$ \\
\bottomrule

\end{tabular}
\end{sc}
\end{table}

In this section, we expand on Section~\ref{ssec:efficient} and provide more details on approximating the search procedure in \eqref{eq:arc-search} in the ARC Algorithm. Recall that the goal is to find the $\ell_p$ shortest distance $\tilde{\zeta}$ between $\vc{x}$ and the hyper-planes that capture the decision boundary of $\tilde{\calR}_m(\vc{x})$, as well as the corresponding unit $\ell_p$ norm direction $\tilde{\vc{g}}$. The $C-1$ hyper-planes are defined by:
\begin{equation}
\calH_j = \left\{\vc{u}\in \reals^D: \tp{\tilde{\vc{w}}_j}\left(\vc{u}-\vc{x}\right) + \tilde{h}_j= 0 \right\}\ \ \ \forall j\in[C]\setminus\{m\}
\end{equation}
where $m \in [C]$ is the label assigned to $\vc{x}$ by $f$, $\tilde{h}_j = \left[f(\vc{x})\right]_m - \left[f(\vc{x})\right]_j$ and $\tilde{\vc{w}}_j = \nabla\left[f(\vc{x})\right]_m - \nabla\left[f(\vc{x})\right]_j$  $\forall j \in [C]\setminus\{m\}$. 
Therefore, in order to obtain $\tilde{\zeta}$ and $\tilde{\vc{g}}$, we find the closest hyper-plane:
\begin{equation}\label{eq:search-a}
    n= \argmin_{{j\in[C]\setminus \{m\}}} \frac{\left|\tilde{h}_j\right|}{\pnorm{\tilde{\vc{w}}_j}{q}}  = \argmin_{{j\in[C]\setminus \{m\}}} \tilde{\zeta}_j
\end{equation}
and then compute:
\begin{equation}
 \tilde{\zeta} = \tilde{\zeta}_n \ \ \ \ \& \ \ \ \  \tilde{\vc{g}} = -\frac{|\tilde{\vc{w}}_n|^{q-1} \odot \text{sgn}(\tilde{\vc{w}}_n)}{\norm{\tilde{\vc{w}}_n}_q^{q-1}}  
\end{equation}

The search in \eqref{eq:search-a} can be computationally demanding for datasets with large number of classes, e.g., $C=100$ for CIFAR-100 or $C=1000$ for ImageNet, as the gradient computation routine required for finding $\tilde{\vc{w}}_j$ is very compute intensive, and has to be executed $C-1$ times, during every iteration $k$ and classifier $f_i$. 
To alleviate this issue, we perform an approximate search procedure. Intuitively, we expect the closest hyper-plane $n$ would correspond to the output logit $[f(\vc{x})]_n$ that is closest to $[f(\vc{x})]_m$. Therefore, instead of searching over all $C-1$ hyper-planes, we only need to search over $1\leq G \leq C-1$ hyper-planes that correspond to the $G$ largest logits $[f(\vc{x})]_j$ $j\neq m$. Specifically, we  propose approximating the search as follows:

\begin{enumerate}
    \item compute the quantities $\tilde{h}_j$ $\forall j \in [C]\setminus\{m\}$ as before
    \item construct the sorted index set $\calJ = \{j_1, ..., j_G\} \subseteq [C]\setminus \{m\}$ such that:
    \begin{equation}
        0<\tilde{h}_{j_1} \leq  \tilde{h}_{j_2} \leq ... \leq \tilde{h}_{j_G} \leq \tilde{h}_{t}  \ \ \  \ \forall t\in [C]\setminus( \{m\} \cup \calJ)
    \end{equation}
    \item search for the closest hyper-plane over the restricted set $\calJ$:
  \begin{equation}\label{eq:search-b}
    n= \argmin_{j\in \calJ} \frac{\left|\tilde{h}_j\right|}{\pnorm{\tilde{\vc{w}}_j}{q}}  = \argmin_{j\in \calJ} \tilde{\zeta}_j
\end{equation}  
\end{enumerate}
The parameter $G$ controls the accuracy-efficiency trade-off of the approximation, and setting $G=C-1$ yields the exact search in \eqref{eq:search-a}.

To demonstrate the efficacy of this approximation, we use this version of ARC for evaluating the robustness of two RECs of ResNet-20's obtained via $\ell_2$ BAT and $\ell_\infty$ BAT on CIFAR-100. We also measure the total evaluation time, i.e., the total time our script requires to evaluate the robustness of the REC using the entire 10000 samples of the CIFAR-100 testset. We use a workstation with a single NVIDIA 1080 Ti GPU and iterate over the testset with a mini-batch size of $256$ for all evaluations. Table~\ref{tab:arclite} reports the corresponding robustness and required evaluation time for different values of $G$. Table~\ref{tab:arclite} demonstrates that using $G=4$ provides the same robustness evaluation as the full ARC algorithm ($G=99$), while requiring less than $4$ minutes to run, as opposed to $50$ minutes required by $G=99$. This massive reduction in evaluation time, while maintaining the fidelity of robustness evaluation, allows us to scale the ARC algorithm to more complex datasets. Therefore, in all of our CIFAR-100 and ImageNet experiments, we will use this version of ARC with $G=4$.

\end{document}